\documentclass{article}

\usepackage{microtype}
\usepackage{graphicx}
\usepackage{booktabs} 

\usepackage{hyperref}
\usepackage[algo2e]{algorithm2e}

\usepackage[accepted]{icml2022}

\usepackage{amsmath}
\usepackage{amssymb}
\usepackage{mathtools}
\usepackage{amsthm}
\usepackage{placeins}

\usepackage{makecell, multirow}
\usepackage{xcolor}
\usepackage{xr}

\DeclareMathOperator*{\argmin}{arg\,min}
\usepackage{ulem}
\usepackage{caption,subcaption}
\usepackage[capitalize]{cleveref}
\usepackage[utf8]{inputenc} 
\usepackage[T1]{fontenc}    
\usepackage{url}            
\usepackage{booktabs}       
\usepackage{amsfonts}       
\usepackage{nicefrac}       
\usepackage{microtype}      
\usepackage{xcolor}         

\usepackage[english]{babel}
\usepackage{fancyhdr}
\usepackage{amssymb,amsthm}
\usepackage{bbm}
\usepackage{enumitem}

\newcommand{\E}{\mathbb{E}}
\newcommand{\Prob}{p} 

\newcommand{\prob}{\Prob}

\newcommand{\ind}{\mathbbm{1}}

\theoremstyle{plain}
\newtheorem{theorem}{Theorem}[section]
\newtheorem{proposition}[theorem]{Proposition}
\newtheorem{lemma}[theorem]{Lemma}

\theoremstyle{definition}

\theoremstyle{remark}

\usepackage[textsize=tiny]{todonotes}

\icmltitlerunning{Structured Stochastic Gradient MCMC}

\begin{document}

\twocolumn[
\icmltitle{Structured Stochastic Gradient MCMC}



\icmlsetsymbol{equal}{*}

\begin{icmlauthorlist}
\icmlauthor{Antonios~Alexos}{equal,cs}
\icmlauthor{Alex~Boyd}{equal,stats}
\icmlauthor{Stephan~Mandt}{cs,stats}
\end{icmlauthorlist}

\icmlaffiliation{cs}{Department of Computer Science, University of California, Irvine, USA}
\icmlaffiliation{stats}{Department of Statistics, University of California, Irvine, USA}

\icmlcorrespondingauthor{Antonios~Alexos}{aalexos@uci.edu}
\icmlcorrespondingauthor{Alex~Boyd}{alexjb@uci.edu}

\icmlkeywords{Machine Learning, ICML}

\vskip 0.3in
]



\printAffiliationsAndNotice{\icmlEqualContribution} 

\begin{abstract}
Stochastic gradient Markov Chain Monte Carlo (SGMCMC) is a scalable algorithm for asymptotically exact Bayesian inference in parameter-rich models, such as Bayesian neural networks. However, since mixing can be slow in high dimensions, practitioners often resort to variational inference (VI). Unfortunately, VI makes strong assumptions on both the factorization and functional form of the posterior. To relax these assumptions, this work proposes a new non-parametric variational inference scheme that combines ideas from both SGMCMC and coordinate-ascent VI. The approach relies on a new Langevin-type algorithm that operates on a "self-averaged" posterior energy function, where parts of the latent variables are averaged over samples from earlier iterations of the Markov chain. This way, statistical dependencies between coordinates can be broken in a controlled way, allowing the chain to mix faster. This scheme can be further modified in a ``dropout'' manner, leading to even more scalability. We test our scheme for ResNet-20 on CIFAR-10, SVHN, and FMNIST. In all cases, we find improvements in convergence speed and/or final accuracy compared to SGMCMC and parametric VI.
\end{abstract}

\section{Introduction}
\label{sec:intro}
\normalem

There has been much recent interest in deep Bayesian neural networks (BNN) due to their reliable confidence estimates and generalization properties \citep{wilson2020bayesian, jospin2020hands,cardelli2019statistical}. BNNs rely on ensemble averages over model parameters typically obtained from Markov chain Monte Carlo (MCMC) algorithms, which contrasts to regular neural networks that depend on a single set of parameters. The sheer size of these models requires scalable MCMC approaches based on inexpensive stochastic gradients, of which stochastic gradient Markov chain Monte Carlo (SGMCMC) algorithms are most widely used \citep{li2016preconditioned, welling2011bayesian, NIPS2013_309928d4}. These algorithms owe their scalability to approximating gradients via mini-batching. 
    
The main downside of using SGMCMC algorithms is their slow mixing rate in high dimensions. A  less ambitious goal than sampling from the true high-dimensional posterior is to approximate the latter by a simpler, typically factorized distribution, as done in variational inference (VI)~\citep{blei2017variational,zhang2018advances}. However, classical VI methods make strong parametric assumptions on the posterior, such as modeling it using product of univariate Gaussian distributions. These distributional assumptions are frequently over-simplistic in high-dimensional models, where the posterior can be highly multi-modal and possibly heavy-tailed. Another downside is that the variational approximation typically underestimates the posterior variance, leading to poorly calibrated uncertainties and overfitting \citep{ormerod2010explaining,giordano2015linear, zhang2018advances}.

In this work, we derive a new \emph{approximate} SGMCMC approach that takes inspiration from \emph{structured} VI \citep{saul1996exploiting,wainwright2008graphical}, a version of VI that maintains selected posterior correlations while breaking others. While our approach remains a sampling algorithm resembling SGMCMC, we speed up the mixing time by systematically breaking pre-specified posterior correlations. Our approach thus adapts the partial factorization aspect from structured mean-field VI, but it remains fully non-parametric and makes no assumptions on the functional form of the approximate posterior. For this reason, we name our approach \emph{structured SGMCMC}.
    
We start by reviewing the functional view on VI, deriving the optimal variational distribution for a given posterior subject to factorization constraints. Our main contribution is then to show that one can sample from this target distribution by running SGMCMC on a modified energy function. This energy function is inspired by the partially-marginalized log joint distribution as encountered in coordinate-ascent VI~\citep{saul1996exploiting,bishop2006pattern}. We can tractably approximate the involved expectations by averaging selected latent variables over samples from the Markov chain. The resulting "self-averaged" energy function defines un-normalized distribution over the remaining coordinates. We can sample from this distribution using Langevin dynamics or any other SGMCMC scheme.

While the resulting posterior approximation can be shown to have favorable properties (e.g., it can be multi-modal and it mixes much faster), the original sampling algorithm is not sufficiently fast for high-dimensional models. Therefore, we provide a more robust and computationally efficient approximation to the procedure that allows for interpolation between regular SGMCMC and our structured SGMCMC by taking inspiration from dropout techniques. 
Both methods are compatible with \emph{any} Markovian SGMCMC algorithm, including Langevin dynamics and stochastic gradient Hamiltonian Monte Carlo.  
  
In sum, our contributions are as follows: \vspace{-0.7em}
\begin{itemize}[leftmargin=12pt,itemsep=1pt]
        
        \item For any differentiable Bayesian model, we specify the optimal approximate posterior (in KL divergence) subject to factorization constraints and show that we can sample from it using a novel self-averaged SGMCMC scheme. 
        \item The new MCMC/VI hybrid allows sampling from a fully joint posterior, a completely factorized posterior, and any in-between. It thereby trades approximation quality for speed. 
        \item We extend this scheme further by making it more scalable with a dropout-inspired approximation. This new scheme has a hyperparameter that enables a smooth interpolation between full SGMCMC and a "mean-field" version where all posterior correlations are broken. 
        \item We show in both small and large scale experiments that our method well approximates posterior marginals and gives improved results over SGMCMC and parametric VI on Resnet-20 architectures on CIFAR-10, Fashion MNIST, and SVHN in terms of runtime and/or final accuracy. Additionally, all code and implementations have been made publicly available.\footnote{ \url{https://github.com/ajboyd2/pytorch_lvi}}
\end{itemize}
Our paper is structured as follows: \cref{sec:related} presents the related work to our proposal, \cref{sec:preliminaries} introduces preliminaries regarding the energy function and the stochastic gradient updates, \cref{sec:method,sec:method_dropout} derive our proposed methods, \cref{sec:experiments} details experiments and their results, and \cref{sec:conclusions} contains our concluding thoughts.

\section{Related Work}
\label{sec:related}

Our work connects both to (stochastic) variational inference 
\citep{bishop2006pattern, hoffman2013stochastic, hoffman2015structured, wang2013variational, tierney1989fully, mackay1992information, bickel2007mathematical, ranganath2014black, ranganath2013adaptive, nalisnick2019dropout, ambrogioni21a, weilbach20a, silvestri2021embedded, papamakarios2021normalizing, liu2016stein} and scalable MCMC \citep{welling2011bayesian, robbins1951stochastic, li2016preconditioned, hoffman2020black, chen2014stochastic, ma2017stochastic, heek2019bayesian, leimkuhler2019partitioned, zhang2019cyclical, mandt2017stochastic, NIPS2013_309928d4}.
For space limitations, we focus on the most related work at the intersection of both topics. 


Among the earliest works to hybridize both approaches was \citep{de2001variational} who constructed a variational proposal distribution in the Metropolos-Hastings step of MCMC. An improved approach to that was introduced in \citep{habib2018auxiliary}, where by introducing low-dimensional auxiliary variables they fit a more accurate approximating distribution. Other related advances to MCMC methods were proposed by \citet{levy2017generalizing} who developed a method to train MCMC kernels with NNs, and \citet{Tongzhou2018, gong2018meta} who leveraged meta learning schemes in SGMCMC methods.

Most recent work focuses on connections between VI and stochastic gradient-based MCMC, or between VI and stochastic gradient descent (SGD). For example, \citet{mandt2016variational,mandt2017stochastic} and \citet{duvenaud2016early} consider SGD as a type of variational inference, but their approaches did not attempt to close the gap to exact MCMC. Other works aim at explicitly interpolating between both methods.  \citet{domke2017divergence} proposes a divergence bound for hybridizing VI and MCMC, essentially by running Langevin dynamics on a tempered evidence lower bound (ELBO).
\citet{salimans2015markov} embody MCMC steps into the variational inference approximation. 
\citet{ahn2012bayesian} improve stochastic gradient Langevin dynamics by leveraging the central limit theorem and using the estimated inverse Fisher information matrix to sample from the approximate posterior distribution. 
\citet{rezende2015variational} interpreted the path of an MCMC algorithm as a variational distribution, and then fitting parameters to tighten a variational bound. Recently, \citet{hoffman2020black} interpreted (parametric) VI as approximate Langevin dynamics and showed that both algorithms have similar transient dynamics. 


In contrast to all these approaches, our method is inspired by coordinate ascent variational inference \citep{bishop2006pattern} but uses Langevin updates to generate samples from the target distribution that respects an imposed independence structure. 

\section{Preliminaries}
\label{sec:preliminaries}

\normalem

Variational inference (VI) differs from MCMC in two regards: (1) it imposes a structured (e.g., fully-factorized) approximation of the posterior for tractability, and (2) it often makes parametric (e.g., Gaussian) assumptions. Is it possible to construct a modified scheme that only relies on the assumption (1), inheriting the non-parametric nature of MCMC while breaking posterior correlations in a controlled manner? We will show how such a scheme can be realized. 

We begin by introducing the setup and common notation. 
Given data $\mathcal{D}=\{(x_i,y_i)\}_{i=1,\dots,N}$, parameters $\theta$, a proper prior distribution $p(\theta)$, and a likelihood $p(\mathcal{D}|\theta) = \prod_{i=1}^N p(y_i|x_i,\theta)$, we represent the Bayesian posterior as a Boltzmann distribution:
\begin{align}
p(\theta|\mathcal{D}) & \propto \exp\{-U(\theta)\}
;\nonumber \\
 U(\theta) & := -\sum_{(x,y)\in \mathcal{D}} \log p(y|x,\theta) - \log p(\theta). \label{eq:energy_dist}
\end{align}
We call $U$ the \emph{posterior energy function}. Note that the posterior is typically intractable due to the normalizing constant. 


Stochastic gradient MCMC (SGMCMC) is a class of scalable MCMC algorithms that can produce posterior samples through gradients on minibatches of data.
These algorithms are largely derived from discretized approximations of continuous-time diffusion processes. Examples of these algorithms include stochastic gradient Langevin dynamics (SGLD) \citep{welling2011bayesian}, preconditioned SGLD (pSGLD) \citep{li2016preconditioned}, and stochastic gradient Hamiltonian Monte Carlo (SGHMC) \citep{chen2014stochastic}.

SGMCMC algorithms use minibatches $\tilde{\mathcal{D}}$ from $\mathcal{D}$ to produce an unbiased estimate of the posterior energy function $U(\theta)$:
\begin{equation}
\begin{aligned}
\hat{U}(\theta;\tilde{\mathcal{D}}) = -\frac{N}{|\tilde{\mathcal{D}}|} \sum_{(x,y)\in\tilde{\mathcal{D}}} \log p(y|x,\theta) - \log p(\theta). \label{eq:energy_approx}
\end{aligned}
\end{equation}
These minibatches enable a sequence of samples from the posterior. For instance, the SGLD update rule is
\begin{equation}
\begin{aligned}
\theta^{(t+1)}  = \theta^{(t)} -\frac{\epsilon_t}{2} \nabla_\theta \hat{U}(\theta^{(t)}; \tilde{\mathcal{D}}_t) + \xi_t, \label{eq:sgld}
\end{aligned}
\end{equation}
where $\xi_t \sim \mathcal{N}(0,\epsilon_t I)$.
Similar rules for pSGLD and SGHMC can be found in the Appendix.
All of these update rules produce a \emph{chain} of samples up to time step $t$ that ultimately form an empirical distribution $\hat p^{(t)}(\theta | \mathcal{D})$. Should the algorithms converge, then $\lim_{t\rightarrow \infty} \hat p^{(t)}(\theta | \mathcal{D}) = p(\theta | \mathcal{D})$.

\section{Structured SGMCMC}
\label{sec:method}

SGMCMC methods produce a fully joint posterior distribution over parameters $\theta$. For models with a large number of parameters, this can lead to various complications due to the curse of dimensionality such as slow convergence times and  unexplored posterior modes. A viable solution is to simplify the sampling task by breaking dependencies in the posterior distribution, leveraging ideas commonly used in VI. 

We begin by reviewing structured variational inference and present a novel SGMCMC-like algorithm that provides samples from a factorized approximate posterior. We then prove that the only stationary distribution of the scheme is a factorized Boltzmann distribution that minimizes KL divergence to the true posterior, subject to the factorization constraint. 

\paragraph{Structured Variational Inference Revisited.} We begin by reviewing structured variational inference \citep{saul1996exploiting} and formulate it as a non-parametric KL minimization problem. To achieve partial factorization, we must first partition $\theta$ into $M>1$ distinct, \emph{mutually independent} groups: $\theta_1, \dots, \theta_M$. This partitioning structure is assumed to be known a priori. We will denote the distribution that respects this partitioning structure as $q(\theta) = \prod_{i=1}^M q_i(\theta_i)$. As common in VI~\citep{zhang2018advances}, we would like $q(\theta)$ to best approximate the true posterior distribution $p(\theta|\mathcal{D})$ in terms of KL-divergence,
\begin{align}
q(\theta) = \argmin_q \mathbb{E}_{\theta\sim q}\left[\log \frac{q(\theta)}{p(\theta|\mathcal{D})}\right]. \label{eq:kl}
\end{align}

To derive a formal solution, we introduce some notation. For any $i \in \{1, ..., M\}$, let $\theta_{\lnot i}$ denote the parameters that complement $\theta_i$, i.e., $\theta = \{\theta_i, \theta_{\lnot i}\}$.
We then define a structured energy function:
\begin{multline}
    U^{(S)}(\theta)=\sum_{i=1}^M U^{(S)}_i(\theta_i),\; \text{ with } U_i^{(S)}(\theta_i) :=\\ \mathbb{E}_{\tilde\theta \sim q} U(\{\theta_i, \tilde\theta_{\lnot i}\}) := - \mathbb{E}_{\tilde\theta \sim q} \log p(\theta_i, \tilde\theta_{\lnot i}, \mathcal{D}).
    \label{eq:US}
\end{multline}
That is, we first define the marginals $U_i^{(S)}(\theta_i)$, where we marginalize $U(\theta)$ with respect to all $q(\theta)$-factors except $q_i(\theta_i)$, and then sum up these marginals to define $U^{(S)}(\theta)$. A similar partial marginalization procedure is carried out for conjugate exponential family distributions in coordinate ascent VI \citep{bishop2006pattern}, where the resulting marginals are used to update the variational distribution in closed form. Note that in most modern applications, conditional conjugacy is violated \citep{zhang2018advances, blei2017variational}. 

We attribute the following theorem to \citep{saul1996exploiting}:

\begin{proposition} \citep{saul1996exploiting}.
The unique solution to the KL minimization problem given in Eq.~\ref{eq:kl} is given by the Boltzmann distribution $q(\theta) \propto  \exp\{-\sum_{i=1}^M U_i^{(S)}(\theta_i)\}$. \emph{Please refer to \cref{proof_proposition1} for the proof.}
\label{thm:kl_min}
\end{proposition}

While the proposition provides a statement of optimality, at this point it is unclear how the theoretically optimal distribution can be realized. This will be discussed next.

\paragraph{Approximating the Energy Function.}
Having a well-defined energy function $U^{(S)}$ (\cref{eq:US}) should in principle allow us to use standard SGMCMC methods to approximate the posterior $q(\theta)$ with samples from a chain of parameter iterates. 
While the marginals  $U_i^{(S)}(\theta_i)$ in \cref{eq:US} do not have a closed-form solution, we can still approximate them, assuming we have a distribution $q$ that we can sample from. In addition, we can sample a mini-batch $\tilde{D}$ to estimate \cref{eq:US} more scalably,
\begin{align}
 \hat{U}^{(S)}(\theta; \tilde{\mathcal{D}}) = \sum_{i=1}^M \mathbb{E}_{\tilde{\theta} \sim q} \hat{U}(\{\theta_i, \tilde{\theta}_{\lnot i}\}; \tilde{\mathcal{D}}), \label{eq:var_sum}
\end{align}
where $\hat{U}(\cdot)$ is defined in \cref{eq:energy_approx}. We stress that \cref{eq:var_sum} is an finite-sample empirical expectation.

\paragraph{Defining the Variational Distribution.} It remains to define how we construct the distribution $q(\tilde \theta)$. 
Assume that we already have a chain of parameter iterates $\{\theta^{(0)}, \cdots, \theta^{(t)}\}$, for example, initialized by some MCMC scheme or consisting just of  $\{\theta^{(0)}\}$. Based on the grouping of parameters, we can define an empirical distribution of marginals $\hat{q}_i$ and a corresponding factorized distribution $\hat{q}$ as
\begin{align}
    \hat{q}(\tilde \theta) := \prod_{i=1}^M \hat{q}_i(\tilde \theta_i); \; \hat{q}_i(\tilde \theta_i) := \text{Unif}(\{\theta^{(0)}_i, \cdots, \theta^{(t)}_i\}).
\end{align}
Thus, a sample from this factorized distribution results in elements from the parameter chain, uniformly chosen at random and \emph{independently} for every group,
\begin{align*}
    {\tilde \theta} \sim \hat{q} \Leftrightarrow {\tilde \theta} = (\theta_1^{(t_1)}, \cdots, \theta_M^{(t_M)}),\; t_j \sim {\rm Unif}(0, \dots, t).
\end{align*}

Using a chain of parameter iterates to define $\hat{q}$ thus allows us to estimate $\hat{U}^{(S)}(\theta; \tilde{D})$ of \cref{eq:var_sum} empirically  
for every timestep $t$ with a Monte Carlo approximation. 
This enables an update step for the next parameter iterate: 
\begin{equation}
\begin{aligned}
\theta^{(t+1)}  = \theta^{(t)} -\frac{\epsilon_t}{2} \nabla_\theta \hat{U}^{(S)}(\theta^{(t)}; \tilde{\mathcal{D}}_t) + \xi_t,\label{eq:s_sgld}
\end{aligned}
\end{equation}
where $\xi_t \sim \mathcal{N}(0,\epsilon_t I)$. We term this algorithm \emph{structured SGLD} (S-SGLD); similar rules for structured variants of pSGLD (S-pSGLD) and SGHMC (S-SGHMC) can be found in the Appendix. The general framework is termed structured SGMCMC (S-SGMCMC). 

\paragraph{Algorithm Summary.} A summary of the algorithm is as follows. We first initialize the parameters $\theta$ and partition them into $M$ groups. The main algorithm evolves in a loop. For every timestep, we sample a mini-batch of the data. Then, for each group $i$ we sample from the empirical distribution $\hat q_{\lnot i}^{(t)}$ and assign the samples as $\tilde{\theta}_{\lnot i}^{(t)}$. These samples, along with the current parameter values $\theta_i^{(t)}$, are used to estimate the marginal energy posterior functions $\hat U_{i}^{(S,t)}$. These are summed across all groups and the result is then derived to execute a SGMCMC update step to produce new samples. A visualization of this can be seen in \cref{fig:sampling_procedure} while the full procedure is detailed in \cref{sec:algo}. 

\begin{figure}
    \centering
    \includegraphics[width=0.95\linewidth]{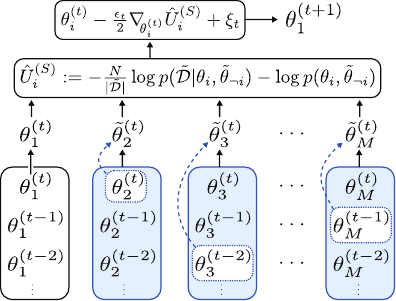}
    \caption{Visualization of sampling procedure for the first of $M$ parameter groups using S-SGLD. The bottom boxes represent the approximate posterior of samples $\hat{q}$---one for each group. The blue boxes are those to be sampled from while the dashed lines indicate which specific values were sampled when estimating the marginal energy for group 1.}
    \label{fig:sampling_procedure}
\end{figure}

\begin{figure*}
    \centering
    \includegraphics[width=0.95\textwidth]{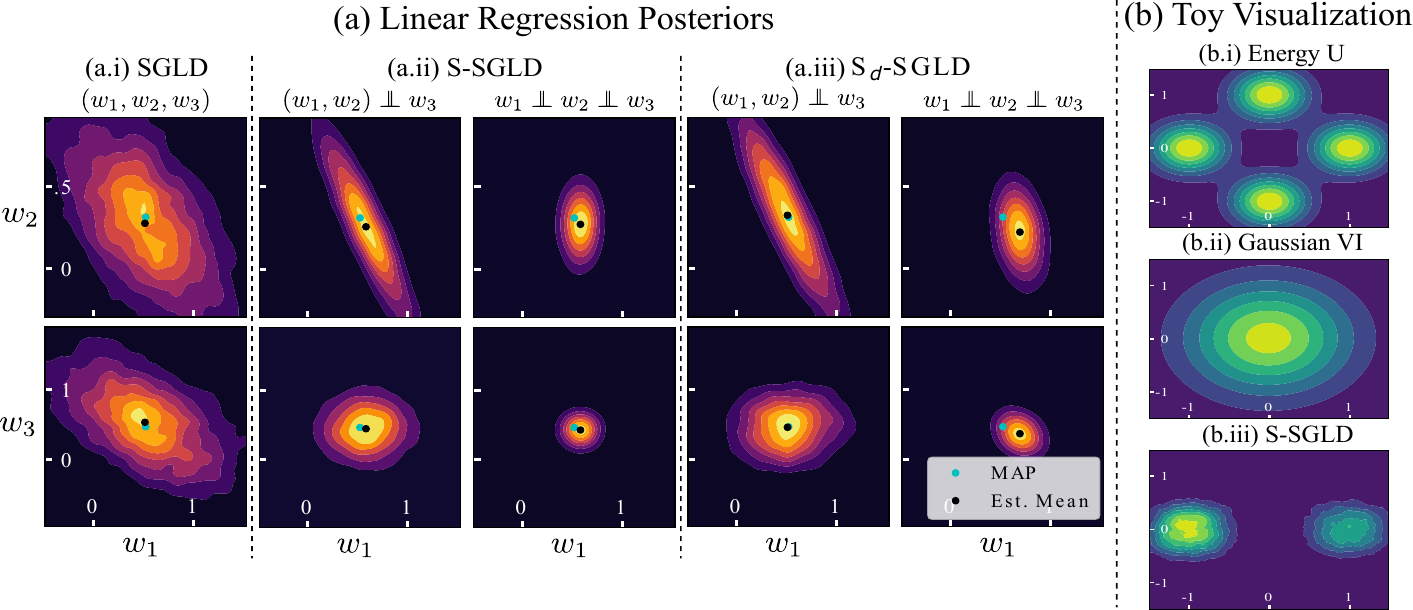}
    \vspace{-0.5em}
    \caption{(a) Sampled posteriors between weights $w_1$ \& $w_2$ (top row) and $w_1$ \& $w_3$ (bottom row) for a three-dimensional linear regression model. From left to right, (a.i) the first column has posterior distributions sampled with SGLD, (a.ii) the second and third columns are sampled from S-SGLD, and (a.iii) the fourth and fifth are sampled from S$_d$-SGLD. The imposed partial parameter factorizations are shown above the plots. (b) An artificial multivariate "posterior"  (b.i) was used as a target for both parametric Gaussian VI (b.ii) and fully factorized S-SGLD (b.iii). S-SGLD captures multiple modes. }
    \label{fig:lin_reg}
\end{figure*}

\paragraph{Visualization in 2D.} An example of S-SGMCMC can be seen in \cref{fig:lin_reg}(a), which features the approximate posterior distributions of a trivariate linear regression model with various independence structures imposed with S-SGLD: (a.i) joint dependence between $w_1, w_2,$ and $w_3$; (a.ii-left) dependence between $w_1$ and $w_2$ but independence between $w_3$ and the other coefficients; (a.ii-right) fully factorized. The bivariate posterior distributions respects the imposed independence structure. Interestingly, the variance shrinks as we induce these factorizations which is a commonly seen artifact when using VI due to the reverse KL divergence. 

\cref{fig:lin_reg}(b) also contrasts the results of S-SGLD (b.iii) and parametric Gaussian VI (b.ii) when employed on a multi-modal posterior distribution (b.i). We can see that in this setting, using S-SGLD the independence structure between variables is maintained while still being able to capture multiple modes---unlike parametric Gaussian VI, making S-SGLD strictly more expressive than parametric approaches.

\paragraph{Correctness of the Scheme.} Assuming that the  chain defined by \cref{eq:var_sum,eq:s_sgld} has reached \emph{some} steady-state $q(\theta_1, …, \theta_M)$ and assuming that all discretization errors can be neglected (i.e., the step size is vanishingly small), we can prove that $q(\theta_1, …, \theta_M)$ must be the Boltzmann distribution from \cref{thm:kl_min} and hence \emph{the optimal non-parametric variational distribution for our given factorization constraint} according to \cref{thm:kl_min}. Note that in practice, an MCMC algorithm reaches a steady-state only asymptotically and not in finite time. In this sense, our arguments should be understood as asymptotic statements.
 
By construction of the algorithm and for a given coordinate $\theta_i$, we pick samples 
from the product of empirical marginals.  
Because we assumed the joint is stationary, all marginals are also stationary. Therefore, the Langevin chain decouples across dimensions since \cref{eq:var_sum} becomes a stationary expectation. Hence, it is sufficient to focus on $\theta_i$ in isolation and prove that its stationary distribution is (up to a normalization constant) given by $\exp\{-U^{(S)}_i (\theta_i)\}$. 

However, the last piece is assured because our scheme results in an unbiased estimate of $U^{(S)}(\theta_i)$, namely $\hat{U}^{(S)}(\theta_i; \tilde{\mathcal{D}})$ in \cref{eq:var_sum}. For Langevin dynamics or any other SGMCMC scheme to converge, this condition is sufficient (assuming negligible discretization errors). The chain of parameter group $\theta_i$ is then guaranteed to converge to  $\exp\{-U^{(S)} (\theta_i)\}$, and the chain of all parameters converges to $\prod_{i=1}^M \exp\{-U^{(S)} (\theta_i)\} = \exp\{- \sum_{i=1}^M U^{(S)} (\theta_i)\}$ (again, dropping normalization constants). The stationary distribution is therefore the one given in \cref{thm:kl_min}.

\section{Structured Dropout SGMCMC}
\label{sec:method_dropout}
While S-SGMCMC can successfully break dependencies between parameter groups, it does suffer computationally due to each parameter update scaling linearly with respect to $M$. This means that for a single new sample of $\theta$, the model's forward pass needs to be computed $M$ different times on the same batch of data $\tilde{\mathcal{D}}$. This can quickly become prohibitively expensive for deep models when $M$ is large. 

Ideally, we would prefer a method that both closely resembles the S-SGMCMC procedure and scales independently from the partitioning scheme. This section presents such a method that achieves this, which we call \emph{structured dropout} SGMCMC (S$_d$-SGMCMC), as well as an informal motivation and derivation of the method. More formal details and a theorem proving both SGMCMC and S-SGMCMC are limiting cases for S$_d$-SGMCMC are given in the Appendix.

The main motivation for this technique can be seen by recognizing that the composition $\{\theta^{(t)}_i, \tilde{\theta}^{(t)}_{\lnot i}\}$ from \cref{eq:var_sum} can be rewritten as a sum of masked values $r\theta^{(t)} + (1-r)\tilde{\theta}^{(t)}$ where $\tilde{\theta}^{(t)}\sim q^{(t)}$ and $r_{j}=\ind(i=j)$ for $i=1,\dots,M$. We can decouple the computational scaling from the number of parameter groups $M$ by replacing the $M$ deterministic masks $r$ with $K$ stochastically sampled masks $\tilde r$.\footnote{$K$ is a hyperparameter that is chosen independent of $M$; however, both $M$ and the distribution of $\tilde r$ largely influence how small $K$ can be due to how they affect the variance of the gradient of the associated posterior energy function.} Doing so results in a slightly different energy function and minibatch loss to optimize:
\begin{multline}
\hat{U}^{(S_d)}(\theta^{(t)}; \tilde{\mathcal{D}}) \approx\\ \frac{M}{K\mathbb{E}\left[\sum_{i=1}^M r_i\right]}\sum_{k=1}^K \hat{U}(\tilde{r}^{(t,k)}\theta^{(t)} + (1-\tilde{r}^{(t,k)})\tilde{\theta}^{(t,k)}; \tilde{\mathcal{D}}) \label{eq:dropout_scale}
\end{multline}
where $\tilde{r}^{(t,k)}$ is the $k^\text{th}$ sample of $\tilde{r}$ for timestep $t$. A formal justification for \cref{eq:dropout_scale} can be found in the Appendix. These energy approximations lead to the following update step for structured dropout variant of SGLD (S$_d$-SGLD):
\begin{equation}
\begin{aligned}
\theta^{(t+1)}  = \theta^{(t)} -\frac{\epsilon_t}{2} \nabla_\theta \hat{U}^{(S_d)}(\theta; \tilde{\mathcal{D}}) + \xi_t, 
\label{eq:sd_sgld}
\end{aligned}
\end{equation}
where $\xi_t \sim \mathcal{N}(0,\epsilon_t I)$.
The corresponding update rules for the structured dropout variants for pSGLD (S$_d$-pSGLD) and SGHMC (S$_d$-SGHMC) are defined in the Appendix. 
The exact procedure for generating samples of the approximate posterior $\hat{q}^{(t)}$ using structured dropout SGMCMC (S$_d$-SGMCMC) can also be found in the Appendix. 

An example of this method (specifically S$_d$-SGLD with $\tilde{r}_i \overset{iid}{\sim} \text{Bernoulli}(0.5)$ and $K=4$) used on a linear regression model can be seen in \cref{fig:lin_reg}(c). Of note, we can see that the dropout variant largely respects the independence structure imposed, but maybe not as strictly as the exact S-SGLD method seen in \cref{fig:lin_reg}(b). Additionally, the posterior variance also seems to have shrunk similarly to S-SGLD when compared against SGLD.

\paragraph{Dropout Analogy.} Should $\tilde{r}_i \overset{iid}{\sim} \text{Bernoulli}(\rho)$, alongside a structure that factorizes by activation components, then the method starts to resemble dropout with rate $\rho$ \citep{srivastava2014dropout}. The main difference being that instead of replacing a parameter value with $0$ it is replaced with a sample from the approximate posterior distribution at time $t$: $\hat{q}^{(t)}$. While a Bernoulli distribution for $\tilde{r}$ is a natural choice, there are other distributions that can be chosen as well. For instance, $\tilde{r}_i \overset{iid}{\sim} \mathcal{N}(0, 1)$ or $\tilde{r}_i \overset{iid}{\sim} \text{Beta}(\alpha, \beta)$ are both viable and can be seen as analogous to Gaussian and Beta-dropout respectively \citep{srivastava2014dropout, liu2019beta}. Our experiments primarily use Bernoulli distributions.
\section{Experiments}
\label{sec:experiments}

\paragraph{Overview} In this section we evaluate our proposed approach on various models and datasets.
\cref{sec:exp-medium} investigates the impact of the variational approximation on the algorithms' mixing and autocorrelation times using a fully-connected network architecture on MNIST \citep{lecun2010mnist}. \cref{sec:exp-large} studies our methods with ResNet-20 \citep{he2016deep} on CIFAR-10 \citep{krizhevsky2009learning}, SVHN \citep{netzer2011reading}, and Fashion MNIST \citep{xiao2017fashion} and compares them for their convergence speed and final accuracy. Our experiments reveal that the chains in our proposed methods converge faster than SGMCMC (and in most cases even parametric VI) while achieving comparable or higher accuracies.  
 Lastly, \cref{sec:ablation} investigates the effects, or lack thereof, that the various partitioning schemes have when utilizing S-SGLD with feed forward neural networks for regression tasks. We also conducted additional experiments, such as hyperparameter optimization, which can be found in the Appendix \cref{sec:hyperparameter_optimization}.

\paragraph{Metrics} 
The primary predictive metric of interest used to evaluate our proposal is classification accuracy. We take the average of an ensemble of 100 models. The weights of every model are sampled from the approximate posterior of a given method. We report the average of individual accuracies produced from each sampled model. 
Additionally, we also monitor the mixing time of the chains of our methods with both integrated autocorrelation time (IAC) \citep{sokal1997monte, goodman2010ensemble} and effective sample size (ESS) \citep{geyer1992practical}.
IAC measures the correlation between samples in a chain and, in turn, describe the inefficiency of a MCMC algorithm. It is an estimate of the number of iterations, on average, for an independent sample to be drawn, given a Markov chain. 
ESS measures how many independent samples would be equivalent to a chain of correlated samples.
Please refer to \cref{sec:exp_details} for more details on these metrics, as well as other experimental setup and implementation specifics.
Higher ESS and lower IAC values can be interpreted as a model having better mixing behavior. 

\subsection{Dropout Rate \& Group Size Investigation}
\label{sec:exp-medium}

This set of experiments aims to study the effects of the number of independent parameter groups (or the amount of allowed posterior correlations) on accuracy and mixing time when using our proposed methods. We compare pSGLD, S-pSGLD, and S$_d$-pSGLD with a $\text{Bernoulli}(\rho)$ masking distribution with dropout rates $\rho\in\{0.1,0.3,0.5\}$ on a fully-connected neural network with 2 hidden layers, with 50 hidden units each, trained and evaluated with MNIST using the standard train and test split.  
The model has 42,200 parameters in total. For S-pSGLD and S$_d$-pSGLD, these parameters are evenly distributed into $M$ groups where $M$ ranges from 4 to 42,200. 
Accuracy, IAC, and ESS are reported in \cref{fig:groups_study} using 100,000 posterior samples after a 10,000 burn in period. More details on the implementation of the model regarding training and evaluation can be found in the Appendix.

As shown in \cref{fig:groups_study}(a), for S-pSGLD we observe that as we increase the number of groups the accuracy drops dramatically whereas S${_d}$-pSGLD's accuracy improves slightly and then remains fairly stable. 
In the best case, S${_d}$-pSGLD achieves an accuracy of 96.3\% with 32 groups and dropout rate of 0.5 which outperforms pSGLD with accuracy of 94.2\%. 
We speculate that the dropout-like behavior is beneficial for regularizing the model (much like normal dropout), improving accuracy across all dropout rates. Similarly, a single sample used for the Monte Carlo estimate in S-SGMCMC may not be enough as the number of groups $M$ increases; however, increasing the number of samples in this scenario is infeasible due to S-SGMCMC scaling as $\mathcal{O}(M)$.

\cref{fig:histogram_mnist} depicts a histogram that compares the IAC values of all dimensions on MNIST between pSGLD and S-pSGLD (with 8192 parameter groups). The IAC values of S-pSGLD are concentrated on the lower end than the pSGLD, which are concentrated on the higher end, further proving the faster mixing of our proposed method.

\cref{fig:groups_study}(b-c) portrays the comparison between number of groups and mixing time metrics IAC and ESS. 
As the number of groups gradually increases, we note that S-pSGLD mixes faster, as does S$_d$-pSGLD to lesser and lesser degrees as $\rho$ increases.
This behavior is to be expected due to Theorem \ref{thm:limit_case}, with S$_d$-pSGLD exhibiting mixing times more similar to pSGLD when $\rho=0.5$ and more similar to S-pSGLD when $\rho=0.1$. 

\begin{figure}
	\centering 
	\includegraphics[width=0.95\linewidth]{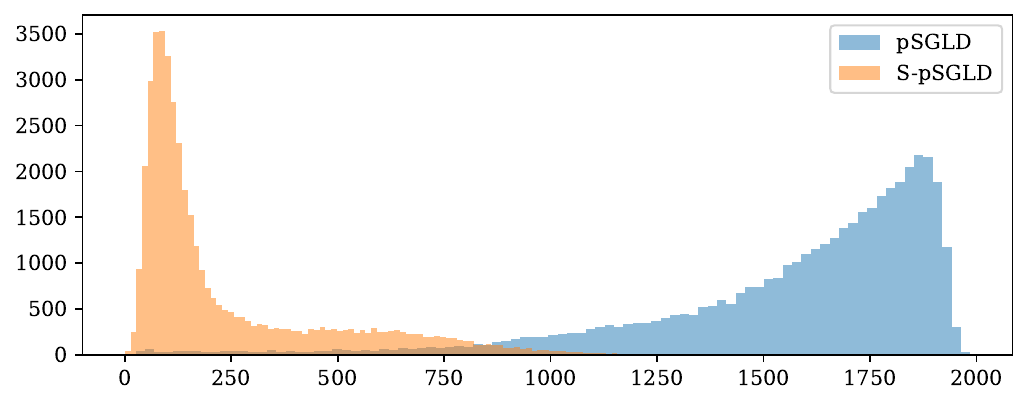}
    \vspace{-1em}
	\caption{Histogram of the IAC values between pSGLD and S-pSGLD (with 8192 parameter groups) on MNIST dataset of each individual parameter. Parameters exhibit much lower IAC values on average when using S-pSGLD.}
	\label{fig:histogram_mnist}
\end{figure}

\begin{figure*}
	\centering 
	\includegraphics[width=0.95\textwidth,height=0.28\textwidth]{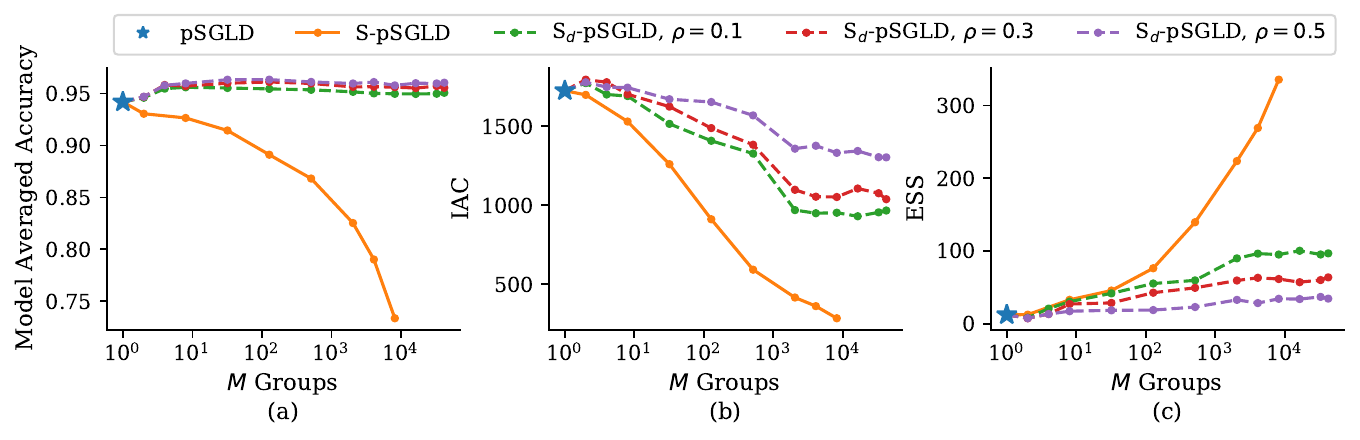}
    \vspace{-1.5em}
	\caption{
	Investigating the effect of the variational approximation's number of parameter groups $M$, where $M=|\theta|$ models a fully-factorized posterior, and $M=1$ captures the full posterior. Experiments were carried out on MNIST with pSGLD as a sampler (see Section~\ref{sec:exp-medium}). We show both non-dropout (S-pSGLD) and dropout versions S${_d}$-pSGLD) of our approach, where we varied the dropout rate $\rho$. For reference, we also include the pSGLD result (at $M=1$). While Figure (a) studies accuracy, Figures (b) and (c) show IAC and ESS, respectively. Our findings are that mixing increases monotonically with $M$ (see Figures (b) and (c)) while accuracy is does not show a strong dependence on $M$ (see Figure (a)). 
    }
	\label{fig:groups_study}
\end{figure*}

\subsection{Systematic Comparison on Real-World Data}
\label{sec:exp-large}

The goal of these experiments is to test the proposed methodology on larger-scale datasets which mimic real-world data: CIFAR-10, SVHN, and FMNIST. We evaluate our methods on performance accuracy and convergence speed. We employ ResNet-20 for SVHN and FMNIST without any data augmentation to assess our methods. For CIFAR10, we employ the same data augmentation process as proposed by \citet{Cubuk_2019_CVPR}. We evaluate the precision of the methods on accuracy over time and their overall mixing time on IAC and ESS with two base algorithms: pSGLD and SGHMC. We limited our scope to models with either fully joint posteriors or fully factorized for efficiency purposes. For the latter, we employed S$_d$-SGMCMC methods with a Bernoulli$(\rho)$ masking distribution, as S-SGMCMC would not be feasible with a number of parameter groups $M$ bigger than 8. We also compare these methods to Gaussian Mean-Field VI (MFVI), VI with inducing weights (MFVI-I), and Deep Ensembles with inducing weights (Ensemble-I) \citep{ritter2021sparse}. The inducing weights are auxiliary variables which have significantly lower dimensions than the weights space and augment that space, making the approach suitable for deep learning architectures. Ensemble-I is \emph{not} a true Bayesian method, but is used in similar settings so it is still included for comparison.


In \cref{cifar10_svhn_acc} we observe how quickly the proposed methods and the baseline SGMCMC methods approach their optimum accuracy over the course of training. In CIFAR10 (\cref{cifar10_svhn_acc}(a)), S$_{d}$-SGMCMC appears to achieve optimal accuracy values much faster than both SGMCMC and VI methods and with both base sampling schemes pSGLD and SGHMC. For SVHN (\cref{cifar10_svhn_acc}(b)), we observe that S$_{d}$-SGMCMC reaches its optimum faster than MFVI-I, SGMCMC and Ensemble-I. Last but not least in FMNIST (\cref{cifar10_svhn_acc}(c)), S$_{d}$-SGMCMC is faster than SGMCMC but slower than the VI methods. In some cases, S$_{d}$-SGMCMC achieves better accuracy values than the baseline and the VI methods, as can be seen in the Appendix. Results on mixing times metrics IAC and ESS can be found in the Appendix, where S$_{d}$-SGMCMC outperforms SGMCMC on all three datasets.

\begin{figure*}
	\centering 
    \includegraphics[width=0.99\textwidth]{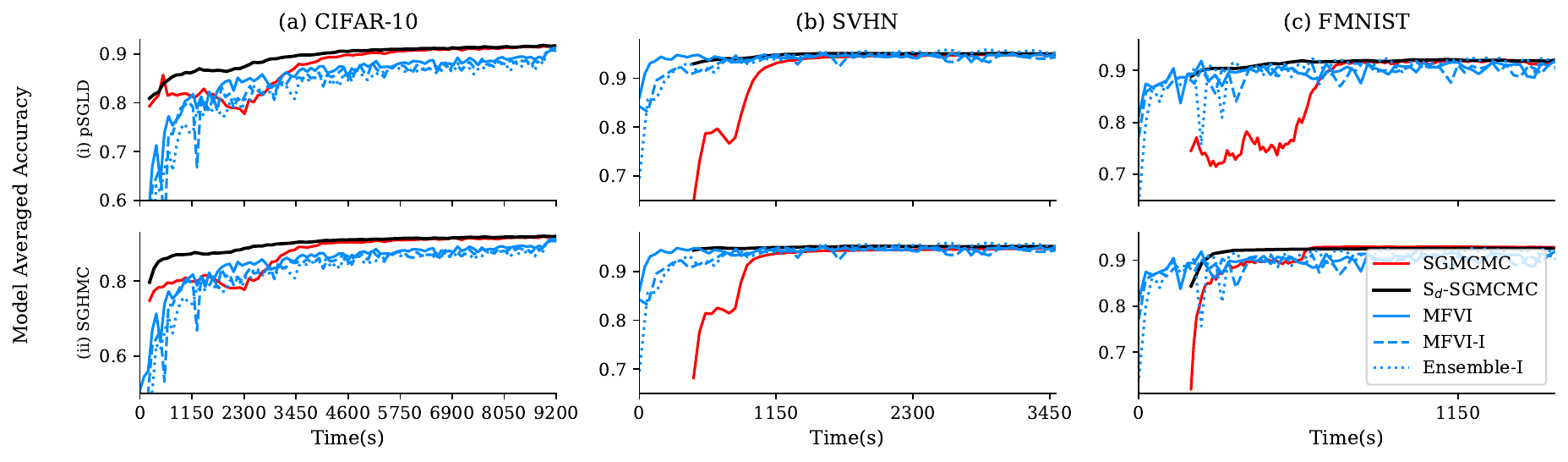}
    \vspace{-1em}
	\caption{Test accuracy as function of training time. We compare  $S_d$-SGMCMC on (a) CIFAR-10, (b) SVHN and (c) FMNIST using ResNet-20 with (i) pSGLD and (ii) SGHMC sampling algorithms and their proposed variational variants for model averaged accuracy. We also compared to MFVI, and MFVI-I and Ensemble-I \citep{ritter2021sparse}.
	SGMCMC results are offset to account for time spent computing \emph{maximum a posteriori} estimates to initialize sampling procedures with.
	Grid search was used to determine optimal hyperparameters for each method. More details can be found in the Appendix.}
	\label{cifar10_svhn_acc}
\end{figure*}

\subsection{Exploring Partitioning Schemes}
\label{sec:ablation}

This part of the study aims to explore the capabilities of the proposed methodology further. Here we explore different parameter partitioning schemes on regression datasets.

Here we present the results with different partitions on various regression datasets. We used 7 different datasets: the wine quality dataset \citep{cortez2009modeling}, the Boston housing dataset \citep{harrison1978hedonic}, the obesity levels dataset \citep{palechor2019dataset}, the Seoul bike-sharing dataset \citep{E2020353,bike_sharing2020}, the concrete compressive strength dataset \citep{yeh1998modeling}, and the airfoil self-noise dataset \citep{brooks1989airfoil}. For the evaluation we chose a simple fully connected network with two layers with 50 neurons each, a Gaussian likelihood, and SGLD for sampling. Performance is measured via mean squared error (MSE). Hyperparameter tuning was performed with different learning rates and the final results are the means with the standard deviations of 5 runs with different seeds.
We do not observe any specific systematic trends on the partitions, apart from the fact that in some cases random partitioning performs better. In that way the use of either random partitioning or the fully-factorized partitioning, where every parameter is in a different group appears to be a valid choice \emph{a priori}; especially the latter since we have noted earlier the faster mixing times associated with this partitioning scheme. More details about the partitioning schemes experiments can be found in the Appendix. 


\begin{table*}
 \centering
 \caption{Investigation of different schemes to partition the parameter space into groups. We study a two-layer fully-connected network on six different regression datasets and show the test mean-squared error. No specific partitioning scheme seems strictly superior to any other one. This suggests that a random partitioning scheme works well for Bayesian neural networks. }
  \resizebox{2\columnwidth}{!}{
\begin{tabular}{l c c c c c c}
 \toprule
  Partition Scheme & Wine & Housing & Obesity & Bike & Concrete & Airfoil\\ 
  \midrule
 \addlinespace[0.1cm]
 Random ($M=3$) & 0.0454±0.001 & 0.0233±0.003 & 0.0232±0.005 & 0.0242±0.001 & 0.0226±0.003 & 0.0454±0.001\\
 By Layer & 0.0494±0.001 & 0.0236±0.003 & 0.0274±0.002 & 0.0247±0.001 & 0.0243±0.003 & 0.0494±0.001\\
 By Neuron & 0.0496±0.001 & 0.0233±0.003 & 0.0262±0.003 & 0.0247±0.001 & 0.0238±0.003 & 0.0496±0.001\\
 Fully-Factorized & 0.0478±0.002 & 0.0236±0.002 & 0.0227±0.002 & 0.025±0.001 & 0.0238±0.001 & 0.0478±0.002\\
\bottomrule
\end{tabular}
}
\label{partitions_table}
\end{table*}


\section{Discussion and Conclusions}
\label{sec:conclusions}

In an attempt to hybridize MCMC and VI, we proposed S-SGMCMC: an approach that produces samples from a structured posterior by running SGMCMC on a self-averaged energy function. 
The resulting chain of parameter iterates becomes asymptotically decoupled across user-specified groups of parameters, resulting in faster mixing times.

For better computational efficiency, we proposed S$_d$-SGMCMC: a further generalization of S-SGMCMC inspired by dropout regularization. 
This dropout extension allows interpolating between an SGMCMC algorithm and its corresponding S-SGMCMC method. 

Our experimental results demonstrate that the proposed methods impose structure over posterior distributions, increase mixing speed of the chains, and result in similar or better ensemble accuracies compared to SGMCMC and parametric variational inference. 
We showed that the proposed approach is compatible with different deep learning architectures, such as ResNet-20, and tested it on CIFAR-10, SVHN, and Fashion MNIST. These evaluations have provided strong empirical evidence for the efficacy of our approach in terms of convergence speed and final accuracies.

Despite its proven capabilities, our proposed methodology does come with some limitations.
Namely, for quick access our methods require keeping chains of samples on the GPU, whereas the baseline SGMCMC methods can save samples to disk. This increases the use of GPU memory and further limits the length of the chains for large-scale experiments; however, thinning the chain and reservoir sampling \citep{wenzel2020good} could be employed to help combat this.

Additionally, S-SGMCMC scales poorly with respect to the number of parameter groups. S$_d$-SGMCMC manages to break this dependency; however, it still requires slightly more compute than SGMCMC per sample but is comparable in wall clock time.

\paragraph{Acknowledgements.} 
We thank Eliot Wong-Toi for feedback on our manuscript, as well as valuable discussions with Sebastian Nowozin, Bastiaan Veeling, and Rodolphe Jenatton. We acknowledge support by the National Science Foundation (NSF) under the NSF CAREER Award 2047418; NSF Grants 1928718, 2003237 and 2007719; the NSF Graduate Research Fellowship under grant number DGE-1839285; the Department of Energy, Office of Science under grant DE-SC0022331, as well as gifts from Intel, Disney, and Qualcomm. This material is in part based upon work supported by the Defense Advanced Research Projects Agency (DARPA) under Contract No. HR001120C0021. Any opinions, findings and conclusions or recommendations expressed in this material are those of the author(s) and do not necessarily reflect the views of DARPA.

\FloatBarrier


\bibliography{icml2022}
\bibliographystyle{icml2022}

\appendix
\section{Solution to Structured VI}
\label{proof_proposition1}

\textbf{\cref{thm:kl_min}} The unique solution to the KL minimization problem given in Eq.~\ref{eq:kl} is given by the Boltzmann distribution $q(\theta) \propto  \exp\{-\sum_{i=1}^M U_i^{(S)}(\theta_i)\}$.
\vspace{-1em}
\begin{proof}
We begin with some preliminaries from the main text. Given data $\mathcal{D}=\{(x_i,y_i)\}_{i=1,\dots,N}$, parameters $\theta$, a proper prior distribution $p(\theta)$, and a likelihood $p(\mathcal{D}|\theta) = \prod_{i=1}^N p(y_i|x_i,\theta)$, suppose we are interested in the corresponding posterior distribution $p(\theta|\mathcal{D})\propto p(\mathcal{D}|\theta)p(\theta)$. A convenient representation of the posterior is as a Boltzmann distribution:
\begin{align}
p(\theta|\mathcal{D})  \propto \exp\{-U(\theta)\}, 
\end{align}
where $U(\theta)  = -\sum_{(x,y)\in \mathcal{D}} \log p(y|x,\theta) - \log p(\theta)$. $U$ is typically referred to as the \emph{posterior energy function}. Note that the posterior distribution is typically intractable due to the normalizing constant. 

We also write the equation for KL divergence from the main text:
\begin{align}
J(q(\theta)) 
&= D_\text{KL}(q(\theta) || p(\theta|\mathcal{D}))  \\
& \equiv \mathbb{E}_{\theta\sim q}\left[\log \frac{q(\theta)}{p(\theta|\mathcal{D})}\right] 
\end{align}

We then rewrite Eq.~\ref{eq:kl} as follows: 
\begin{flalign}
    J&(q(\theta)) = \E_{\theta \sim q}\left[\log q(\theta)\right] - \E_{\theta \sim q}\left[\log \prob(\theta, \mathcal{D})\right] + C\\
    =& \E_{\theta_i \sim q_i}\left[\log q_i(\theta_i)\right] + \sum_{i\neq j}\E_{\theta_j \sim q_j}\left[\log q_j(\theta_j)\right] - \notag \\
    &\;\; \int \log p(\theta, \mathcal{D}) q_i(\theta_i)d\theta_i \prod_{i\neq j} q_j(\theta_j)d\theta_j + C 
\end{flalign}
for some $i \in \{1, \dots, M\}$ where $\lnot i := \{1, \dots, M\} \setminus \{i\}$ and $C=\log p(\mathcal{D})$. In order to find the optimal distribution that respects the factorization constraints imposed between parameter groups, we need to minimize this functional over $q$ --- or rather every $q_i$. This is done by taking the functional derivative of $J$ with respect to $q_i$, setting it equal to zero, and solving for $q_i$:
\begin{align}
    \frac{\delta J(q(\theta))}{\delta q_i(\theta_i)} & = \int \log p(\theta, \mathcal{D}) \prod_{i\neq j} q_j(\theta_j)d\theta_j - \notag \\
    & \quad\;\; 1 - \log q_i(\theta_i) := 0 \\ 
    \implies \log q_i(\theta_i) & = \E_{\tilde{\theta}_{\lnot i} \sim q_{\lnot i}}\left[\log p(\theta_i, \tilde{\theta}_{\lnot i}, \mathcal{D})\right] - 1 \\
    \implies q_i(\theta_i) & \propto \exp\left\{\E_{\tilde{\theta}_{\lnot i} \sim q_{\lnot i}}\left[\log p(\theta_i, \tilde{\theta}_{\lnot i}, \mathcal{D})\right]\right\}. 
\end{align}
By defining the energy $U^{(S)}_i(\theta_i) = -\E_{\tilde{\theta}_{\lnot i} \sim q_{\lnot i}}\left[\log p(\theta_i, \tilde{\theta}_{\lnot i}, \mathcal{D})\right]$, we realize that by minimizing the KL-divergence in Eq.~\ref{eq:kl}, the approximate posterior distribution $q=\prod_{i=1}^M q_i$ takes the form of a Boltzmann distribution as in Eq.~\ref{eq:energy_dist} with $U^{(\text{S})}(\theta)=\sum_{i=1}^M U^{(\text{S})}_i(\theta_i)$.

It remains to be shown that the solution is unique. To this end, we refer to the convexity of the KL divergence in function space \citep{cover1991entropy}. This implies that the stationary point of the KL is indeed a global optimum and unique. \qedhere
\end{proof}

\section{Deriving $U^{(S)}$}

With just a slight shift in perspective, it is actually possible to further generalize $U^{(S)}$ (and consequently S-SGMCMC) to produce a broader class of approximate sampling algorithms.
This is done by first noting that $U^{(S)}$ can be represented with a scaled double-expectation:
\begin{align}
& U^{(S)}(\theta) = -\frac{M}{\E_{r\sim p^{(S)}}\left[\sum_{i=1}^M r_i\right]} \times \notag \\
&\quad\;\;E_{r\sim p^{(S)}}E_{\tilde{\theta} \sim q} \left[\log p(r\theta + (1-r)\tilde{\theta}, \mathcal{D})\right] \label{eq:u_exp}
\end{align}
where $p^{(S)}(r) = \text{Cat}(r; M^{-1},\dots,M^{-1})$ and $(r\theta + (1-r)\tilde{\theta})_i$ is equal to $\theta_i$ if $r_i = 1$ and $\tilde{\theta}_i$ otherwise for $i=1,\dots,M$. Note that this is constructed in this manner specifically so that $U^{(S)}$ remains differentiable with respect to $\theta$. Also note that though the denominator appears superfluous as $\E_{r\sim p^{(S)}}[\sum_{i=1}^M r_i] = 1$, it is necessary for certain theoretic properties, as seen in \cref{thm:limit_case}.

By replacing $p^{(S)}$ with a more flexible distribution, we can further generalize and encapsulate different energy functions to sample from. One such choice is $p^{(S_d)}(r;\rho) :\propto \prod_{i=1}^M\text{Bern}(r_i;\rho)\ind(\sum_{i=1}^M r_i > 0)$ with $\rho \in (0,1)$.\footnote{Other choices of distribution that are well justified include any with support over $[0,1]^M$ and with measure 0 over $\{0\}^M$. Exploring the effects these distributions have are an interesting line of future inquiry.} 
Substituting $p^{(S)}$ for $p^{(S_d)}$ in \cref{eq:u_exp} yields a new energy function that we will refer to as $U^{(S_d)}$.
We note that this choice in distribution leads to a dropout-like behavior \citep{nalisnick2019dropout,srivastava2014dropout}, where the composition of model parameters as $r\theta+(1-r)\tilde{\theta}$ leads to each parameter group $\theta_i$ having a probability of approximately $\rho$ to be used in a prediction and a $(1-\rho)$ probability of being replaced by $\tilde{\theta}_i$ from the approximate posterior (in traditional dropout, $\theta_i$ would instead be replaced with 0).
Likewise, we will denote methods that use this energy function for sampling as structured dropout SGMCMC (S$_d$-SGMCMC) with different variants all sharing the same S$_d$ prefix (e.g. S$_d$-SGHMC).

In practice, the double-expectation in $U^{(S_d)}$ is jointly approximated using a Monte Carlo estimate with $K$ samples. This leads to \cref{eq:dropout_scale} in the main paper.
We note that by approximating $U^{(S_d)}$ in this way, computing a gradient no longer scales on the order of $\mathcal{O}(M)$, but rather $\mathcal{O}(K)$. This means that the choice of structure imposed on the posterior distribution remains independent of computing resources. As such, configurations with large amounts of parameter groups are typically only feasible when using S$_d$-SGMCMC as S-SGMCMC would use too much memory and/or compute per sample. 

\section{Limiting Cases for S$_d$-SGMCMC}

\begin{proposition}\label{thm:limit_case}
For a given set of parameters $\theta$ partitioned into $M$ groups, under minor assumptions (i) $U^{(S_d)} \rightarrow U$ as $\rho \rightarrow 1$ and (ii) $U^{(S_d)} \rightarrow U^{(S)}$ as $\rho \rightarrow 0$. Thus, distributions approximated by S$_d$-SGMCMC lie on a continuum with those generated by S-SGMCMC at one extreme and with those from SGMCMC at the other.
\end{proposition}
\vspace{-1em}
\begin{proof}
Assume an arbitrary $\theta$, $\mathcal{D}$, $n\in \mathbb{N}$, and that $\E_{\tilde{\theta} \sim q}\left[\log p(r\theta + (1-r)\tilde{\theta}, \mathcal{D})\right]$ exists for $r \in \mathcal{R}$. As an aside, this proof assumes that $p^{(S_d)}(r;\rho) :\propto \prod_{i=1}^M\text{Bern}(r_i;\rho)\ind(\sum_{i=1}^M r_i > 0)$ with $\rho \in (0,1)$; however, the theorem still holds an arbitrary $p^{(S_d)}$ so long as the mean approaches 1 and variance approaches 0 as $n\rightarrow \infty$.  

(i) Let $r^{(n)} \sim p^{(S_d)}(\rho_n)$ where $\forall_n \rho_n \in (0,1)$ and $\rho_n \rightarrow 1$.
It follows that $r^{(n)} \rightarrow \{1\}^M$ as $n \rightarrow \infty$ in distribution (see \cref{lemma:1} in Supplement). Due to bounded and finite support $\mathcal{R}$, we find the following:
\begin{align}
& U^{(S_d)}(\theta)  = -\frac{M}{\E_{r\sim p^{(S_d)}}[\sum_{i=1}^M r_i]} \times \notag \\
&\quad\;\;\sum_{r\in\mathcal{R}} p^{(S_d)}(r;\rho_n) \E_{\tilde{\theta} \sim q}\left[\log p(r\theta + (1-r)\tilde{\theta}, \mathcal{D})\right] \\
& \rightarrow -\frac{M}{M} \sum_{r\in\mathcal{R}} \ind(\forall_i r_i = 1) \E_{\tilde{\theta} \sim q}\left[\log p(\theta, \mathcal{D})\right] \text{ as } n \rightarrow \infty \\
& = -\log p(\theta, \mathcal{D}) = U(\theta)
\end{align}

(ii) Let $r^{(n)} \sim p^{(S_d)}(\rho_n)$ where $\forall_n \rho_n \in (0,1)$ and $\rho_n \rightarrow 0$.
It follows that $r^{(n)} \rightarrow r \sim \text{Cat}(M^{-1}, \dots, M^{-1})$ as $n \rightarrow \infty$ in distribution (see \cref{lemma:2} in Supplement). Due to bounded and finite support $\mathcal{R}$, we find the following:
\begin{align}
& U^{(S_d)}(\theta) = -\frac{M}{\E_{r\sim p^{(S_d)}}[\sum_{i=1}^M r_i]}\times \notag \\
&\quad\;\; \sum_{r\in\mathcal{R}} p^{(S_d)}(r;\rho_n) \E_{\tilde{\theta} \sim q}\left[\log p(r\theta + (1-r)\tilde{\theta}, \mathcal{D})\right] \\
& \rightarrow -\frac{M}{1} \sum_{r\in\mathcal{R}} \frac{\ind(\sum_{i=1}^M r_i = 1)}{M}\times \notag \\
&\quad\;\;\E_{\tilde{\theta} \sim q}\left[\log p(r\theta+(1-r)\tilde{\theta}, \mathcal{D})\right] \text{ as } n \rightarrow \infty \\
& = -\sum_{i=1}^M \E_{\tilde{\theta} \sim q}[\log p([\theta_i,\tilde{\theta}_{\lnot i}, \mathcal{D})] = U^{(S)}(\theta)
\end{align}
\qedhere
\end{proof}

For both \cref{lemma:1,lemma:2}, let
\begin{align}
p^{(S_d)}(r;\rho) = \frac{\rho^{\sum_{i=1}^M r_i}(1-\rho)^{M-\sum_{i=1}^M r_i}}{1-(1-\rho)^M}\times \notag \\
\ind(\forall_i r_i \in \{0,1\})\ind\left(\sum_{i=1}^M r_i > 0\right)
\end{align}

\begin{lemma}
For $r^{(n)} \sim p^{(S_d)}(\rho_n)$, $\rho_n \in (0,1)$ and $n\in \mathbb{N}$, if $\rho_n \rightarrow 1$ as $n \rightarrow \infty$ then $r^{(n)} \rightarrow r\sim \delta(\{1\}^M)$ in distribution as $n\rightarrow \infty$.
\label{lemma:1}
\end{lemma}
\begin{proof}
\begin{align}
p^{(S_d)}(r=\{1\}^M;\rho_n) & = \frac{\rho_n^{M}(1-\rho_n)^{0}}{1-(1-\rho_n)^M} \\
& \rightarrow 1 \text{ as } n \rightarrow \infty \\
\implies r^{(n)} & \rightarrow \delta(\{1\}^M) \text{ in distribution.}
\end{align}
\qedhere
\end{proof}
\begin{lemma}
For $r^{(n)} \sim p^{(S_d)}(\rho_n)$, $\rho_n \in (0,1)$ and $n\in \mathbb{N}$, if $\rho_n \rightarrow 0$ as $n \rightarrow \infty$ then $r^{(n)} \rightarrow r\sim \text{Cat}(M^{-1},\dots,M^{-1})$ in distribution as $n\rightarrow \infty$.
\label{lemma:2}
\end{lemma}
\begin{proof}
Let $i\in\{1,\dots,M\}$.
\begin{align}
p^{(S_d)}&(r_i=1,r_{\lnot i}=0;\rho_n)  = \frac{\rho_n(1-\rho_n)^{M-1}}{1-(1-\rho_n)^M} \\
& \stackrel{H}{=} \frac{(1-\rho_n)^{M-1} + \rho_n(M-1)(1-\rho_n)^{M-2}}{M(1-\rho_n)^{M-1}} \\
& \rightarrow \frac{1}{M} \text{ as } n \rightarrow \infty
\end{align}
where $\stackrel{H}{=}$ indicates utilizing $\text{l'H\^{o}spital's Rule}$.
Since the resulting probabilities sum to 1, this implies that $r^{(n)} \rightarrow r \sim \text{Cat}(M^{-1},\dots,M^{-1})$ in distribution as $n \rightarrow \infty$.
\qedhere
\end{proof}

\section{Deriving $U^{(S_d)}$}
To derive $U^{(S_d)}$, we must first start with a shift in perspective on how $U^{(S)}$ is represented. We will rewrite the function in the following way:
\begin{align}
    & U^{(S)}(\theta) = -\sum_{i=1}^M \E_{\theta_{\lnot i} \sim q_{\lnot i}}[\log p([\theta_i, \theta_{\lnot i}], \mathcal{D})] \\
    & = -\frac{M}{\E_{r\sim p^{(S)}}[\sum_{i=1}^M r_i]} \times \notag \\
    & \quad\quad \E_{r\sim p^{(S)}}\E_{\tilde{\theta} \sim q}\left[\log p(r\theta + (1-r)\tilde{\theta}, \mathcal{D})\right]
\end{align}
where $p^{(S)}$ is a $M$-dimensional categorical distribution with uniform weights $M^{-1}$ and $p(r\theta + (1-r)\tilde{\theta}, \mathcal{D})$ is the joint probability of parameters taking values of $r\theta + (1-r)\tilde{\theta}$ and data $\mathcal{D}$.\footnote{$r\theta + (1-r)\tilde{\theta}$ is a slight abuse of notation that is meant to represent masking out $\theta_i$ when $r_i=0$ and masking out $\tilde{\theta}_i$ when $r_i=1$.}

We note that changing the distribution of $r$ leads to different energy functions to sample from. One such choice is to have $p^{(S_d)}(r; \rho) \propto \rho^{\sum_{i=1}^M r_i}(1-\rho)^{M-\sum_{i=1}^M r_i}\ind(\forall_i r_i \in \{0,1\})\ind(\sum_{i=1}^M r_i > 0)$ for $\rho\in(0,1)$. Note that this is identical to $r_i \stackrel{iid}{\sim} \text{Bernoulli}(\rho)$ conditional to $\sum_{i=1}^M r_i > 0$. Let the support of $p^{(S_d)}$ be denoted as $\mathcal{R}=\{0,1\}^M \setminus \{0\}^M$. This leads to the following energy function:
\begin{align}\label{eq:drop_energy_func}
& U^{(S_d)}(\theta) = -\frac{M}{\E_{r\sim p^{(S_d)}}[\sum_{i=1}^M r_i]} \times \notag \\
& \quad\quad \E_{r\sim p^{(S_d)}}\E_{\tilde{\theta} \sim q}\left[\log p(r\theta + (1-r)\tilde{\theta}, \mathcal{D})\right].  
\end{align}
In practice, a few approximations are made to compute the corresponding $U^{(S_d)}$. Firstly, we approximate $p^{(S_d)}$ with an $M$-dimensional $\text{Bernoulli}(\rho)$ distribution as the difference is minute when $M\rho$ is large. Secondly, the outer expectation in \cref{eq:drop_energy_func} is approximated with a Monte Carlo estimate of $K$ samples. The inner expectation is also approximated with a Monte Carlo estimate using the latest approximate posterior $\hat{q}^{(t)}$. However, just like for S-SGMCMC, only a single sample is used. This further leads to:
\begin{align}
& U^{(S_d)}(\theta^{(t)}; \tilde{\mathcal{D}}) \notag \\
& = -\frac{1}{K\rho} \sum_{k=1}^K U(r^{(t,k)}\theta^{(t)} + (1-r^{(t,k)})\tilde{\theta}^{(t,k)}; \tilde{\mathcal{D}})
\end{align}

\section{Algorithm for S-SGMCMC and S$_d$-SGMCMC} \label{sec:algo}
The procedures for S-SGMCMC and S$_d$-SGMCMC can be seen in Algorithms \ref{alg:vsgmcmc} and \ref{alg:vdsgmcmc}. 

\begin{algorithm}[H]
  \SetAlgoNoLine
  \DontPrintSemicolon
  \KwIn{Initial sample $\theta^{(0)}$; parameter partitions $\theta_1,\dots,\theta_M$; 
  step sizes $\{\epsilon_t\}_{t=0,\dots,T-1}$.}
  \KwOut{$\hat{q}^{(T)}(\theta) := \{\theta^{(t)}\}_{t=1,\dots,T}$}
  \For{$t=0$ \KwTo $T-1$}{
    Sample minibatch $\tilde{\mathcal{D}}^{(t)} \subset \mathcal{D}$\;\\
    \For{$i=1 \text{ to } M$}{
        $\text{Sample } \tilde{\theta}_{\lnot i}^{(t)} \sim \hat{q}_{\lnot i}^{(t)}$\;\\
        $\!\!\!\!\!\!\!\!\!\!\!\!\!\!\!\!\!\!\hat{U}_i^{(S, t)} = \hat{U}([\theta^{(t)}_i, \tilde{\theta}^{(t)}_{\lnot i}]; \tilde{\mathcal{D}}^{(t)})$\; 
    }
    $\nabla_\theta \hat{U}^{(S, t)} = \sum_{i=1}^M \nabla_\theta \hat{U}_i^{(S, t)}$\; \newline
    $\theta^{(t+1)} =$ \texttt{SGMCMC\_step}$(\theta^{(t)}, \nabla_\theta \hat{U}^{(S, t)},\epsilon_t)$
  }
  \Return $\hat{q}^{(T)}(\theta)$
  \caption{S-SGMCMC}
  \label{alg:vsgmcmc}
\end{algorithm}

\begin{algorithm}[H]
  \SetAlgoNoLine
  \DontPrintSemicolon
  \KwIn{Initial sample $\theta^{(0)}$; parameter partitions $\theta_1,\dots,\theta_M$; 
  step sizes $\{\epsilon_t\}_{t=0,\dots,T-1}$; masking distribution $p(r)$.}
  \KwOut{$\hat{q}^{(T)}(\theta) := \{\theta^{(t)}\}_{t=1,\dots,T}$}
    \For{$t = 0$ to $T-1$}{ 
    Sample minibatch $\tilde{\mathcal{D}}^{(t)} \subset \mathcal{D}$\;\\
    \For{$k=1$ to $K$}{ 
        Sample masks $r^{(t,k)}_{1}, \dots, r^{(t,k)}_{M} \sim p(r)$\;\\
        Sample $\tilde{\theta}^{(t,k)} \sim \hat{q}^{(t)}$\;\\
        $\theta^{(t,k)} = [r^{(t,k)}_{i}\theta^{(t)}_i + (1-r^{(t,k)}_{i})\tilde{\theta}_i^{(t,k)}]_{i=1,\dots,M}$\;\\
        $\hat{U}_k^{(S_d, t)} = \hat{U}(\theta^{(t,k)}; \tilde{\mathcal{D}}^{(t)})$ \; 
    }
    ${\nabla}_\theta \hat{U}^{(S_d, t)} = \frac{M}{K\mathbb{E}_{r\sim p(r)}\left[\sum_{i=1}^M r_i\right]} \sum_{k=1}^K \nabla_\theta \hat{U}_k^{(S_d, t)}$ \;\\
    $\theta^{(t+1)} = \texttt{SGMCMC\_step}(\theta^{(t)}, {\nabla}_\theta \hat{U}^{(S_d, t)},\epsilon_t)$\;
}
  \Return $\hat{q}^{(T)}(\theta)$
  \caption{S$_d$-SGMCMC}
  \label{alg:vdsgmcmc}
\end{algorithm}


\section{SGMCMC Update Rules}

The update rules for SGLD, pSGLD, and SGHMC are defined as follows:
\begin{align}
\text{SGLD} \;\; \theta^{(t+1)} & = \theta^{(t)} -\frac{\epsilon_t}{2}\nabla_\theta \hat U(\theta^{(t)}) + \mathcal{N}(0,\epsilon_t I) \label{eq:sup_sgld}\\
\text{pSGLD} \;\; \theta^{(t+1)} & = \theta^{(t)} -\frac{\epsilon_t}{2}\Bigg[R(\theta^{(t)})\nabla_\theta \hat U(\theta^{(t)}) + \notag \\
& \;\;\; \sum_\theta \nabla_\theta R(\theta^{(t)})\Bigg] + \mathcal{N}(0, \epsilon_t R(\theta^{(t)})) \label{eq:sup_psgld}\\
\text{SGHMC} \;\; \theta^{(t+1)} & =\theta^{(t)} + \epsilon_t M^{-1}m^{(t+1)} \\
m^{(t+1)} & = (1-\gamma\epsilon_t M^{-1})m^{(t)} - \epsilon_t  \nabla_\theta \hat U(\theta^{(t)}) + \notag \\ & \;\;\; \mathcal{N}(0, 2\gamma-\epsilon_t\hat V(\theta^{(t)})) \label{eq:sup_sghmc}
\end{align}
where $\epsilon_t$ is the step size at time step $t$, $R(\cdot)$ and $M$ are preconditioners, $\gamma \geq 0$ is a friction term, and $\hat V(\cdot)$ is an estimate of the covariance induced by the stochastic gradient.\footnote{Note that we abuse notation in Eqs.~\ref{eq:sup_sgld}-\ref{eq:sup_sghmc} where the addition of $\mathcal{N}(\mu,\Sigma)$ denotes the addition of a normally distributed random variable with mean $\mu$ and covariance $\Sigma$.} 

The update rules for the S-SGMCMC variants are similarly defined as Eqs.~\ref{eq:sup_sgld}-\ref{eq:sup_sghmc} but all instances of $\hat{U}(\theta^{(t)})$ are replaced with $\hat U^{(S)}(\theta^{(t)})$. Likewise, replacing with $\hat U^{(S_d)}(\theta^{(t)})$ yields the S$_d$-SGMCMC variants.
\section{Ablation Study}

This subsection aims to further explore the capabilities of the proposed methodology. More specifically, we experiment with various parameter partitions.

\paragraph{Parameter Partitions.} We tested our proposal with four partitioning schemes on a 2 layer with 50 neurons fully connected network on a regression task. The partitioning schemes that we used are the following: (a) the parameters are split into 3 groups randomly, (b) the parameters are split by layer(3 layers, 1 input and 2 hidden), (c) by activating neurons inside the layers and (d) every parameter belongs in each own group. We used 7 different datasets: the wine quality datset \citep{cortez2009modeling}, the Boston housing dataset \citep{harrison1978hedonic}, the obesity levels dataset \citep{palechor2019dataset}, the Seoul bike-sharing dataset \citep{E2020353,bike_sharing2020}, the concrete compressive strength dataset \citep{yeh1998modeling}, and the airfoil self-noise dataset \citep{brooks1989airfoil}. Every dataset was split into 75\% training data, 10\% validation data, and 15\% test data. We trained the model on training set and validated it in the validation set with an early stoppage. For every dataset and every partitioning scheme we used the learning rates: 1e-3,1e-4,1e-5,1e-6,1e-7 for hyperparameter tuning. For each combination of partition and dataset, we chose the learning rate that provides the best accuracy score on the test set. In this case, as an accuracy score, we used the Mean Squared Error. The final learning rates that we used are presented in \cref{learning_rates}.

\begin{table*}
 \centering
 \caption{Best Learning Rates for various partitioning schemes on multiple regression datasets.}
 \begin{tabular}{l c c c c c c}
 \hline
 \addlinespace[0.2cm]
  Partition Scheme & Wine & Housing & Obesity & Bike & Concrete & Airfoil\\ 
  \addlinespace[0.2cm]
 \hline
 \addlinespace[0.1cm]
 Random ($M=3$) & 1e-3 & 1e-5 & 1e-5 & 1e-4 & 1e-5 & 1e-4\\
 By Layer & 1e-3 & 1e-5 & 1e-4 & 1e-4 & 1e-5 & 1e-4\\
 By Neurons & 1e-3 & 1e-5 & 1e-4 & 1e-4 & 1e-5 & 1e-5\\
 Fully-Factorized & 1e-5 & 1e-4 & 1e-5 & 1e-4 & 1e-4 & 1e-3\\
\addlinespace[0.1cm]
\hline
\addlinespace[0.2cm]
\end{tabular}
\centering
\label{learning_rates}
\end{table*}

\paragraph{Mixing Time Comparisons on Real-World Data.} We further validated our findings from \cref{sec:exp-large} by evaluating the IAC and ESS on larger datasets using various methods. Both pSGLD and SGHMC were used as base methods in conjunction with S$_d$-SGMCMC using a Bernoulli masking distribution. IAC and ESS were calculated for these methods using the latest 5,000 samples after sampling for 300 epochs; the results of which can be found in \cref{tab:cifar_svhn_results}.

For all three datasets, we see that S$_d$-SGMCMC with every parameter in a different group mixes the fastest against all other methods.

\begin{table}
\centering
\captionof{table}{IAC and ESS metrics for CIFAR-10, SVHN, and FMNIST with various methods. Subscripts after method names refers to number of equally sized parameter groups, with $|\theta|$ meaning every parameter belongs to its own group. Best results are bolded.}
\begin{tabular}{@{$\;$}l@{$\;\!\!$}c@{$\;$}c@{$\;$}c@{$\;$}c@{$\;$}c@{$\;$}c@{$\;$}c@{$\;$}c@{}}  
\toprule
& \multicolumn{2}{c}{CIFAR-10} & \hspace{0.01em} & \multicolumn{2}{c}{SVHN} & \hspace{0.01em} & \multicolumn{2}{c}{FMNIST}\\
\cmidrule{2-3}\cmidrule{5-6}\cmidrule{8-9}
Method & IAC${\scriptstyle \;\!\downarrow}$ & ESS${\scriptstyle \;\!\uparrow}$ &  & IAC${\scriptstyle \;\!\downarrow}$ & ESS${\scriptstyle \;\!\uparrow}$ &  & IAC${\scriptstyle \;\!\downarrow}$ & ESS${\scriptstyle \;\!\uparrow}$\\ 
\midrule
pSGLD & 831 & 6.89  & & 839 & 6.82 & & 774 & 7.16\\
S$_d$-pSGLD$_{|\theta|}$ & \textbf{620} & \textbf{7.22} & & \textbf{803} & \textbf{7.00} & & \textbf{737} & \textbf{7.57}
\vspace{0.35em} \\
SGHMC & 832 & 6.84 & & 858 & 6.59 & & 920 & 5.72\\
S$_d$-SGHMC$_{|\theta|}$ & 751 & 6.26 & & 827 & 6.56 & & 923 & 5.70\\
\bottomrule
\end{tabular}

\end{table}

\begin{table}
\centering
\caption{Classification accuracy for CIFAR-10, SVHN, and FMNIST with various methods. Subscripts after method names refers to number of equally sized parameter groups, with $|\theta|$ meaning every parameter belongs to its own group. Best results are bolded.}
\begin{tabular}{cccc}
\toprule
Method & CIFAR-10 & SVHN & FMNIST\\
\midrule
pSGLD & 92.43 & 94.72 & 91.14\\
S$_d$-pSGLD$_{|\theta|}$ & 92.51 & 94.82 & 92.47\\
SGHMC & 92.35 & 94.73 & 92.81\\
S$_d$-SGHMC$_{|\theta|}$ & 92.54 & 94.65 & 92.71\\
MFVI & 91.52 & 94.79 & 92.24\\
MFVI-I & 91.91 & 95.91 & 92.74\\
Ensemble-I & \textbf{92.68} & \textbf{96.01} & \textbf{93.45}\\
\bottomrule
\end{tabular}
\label{tab:cifar_svhn_results}
\end{table}

\section{Hyperparameter Optimization}
\label{sec:hyperparameter_optimization}

We also tested the proposed method for hyperparameter optimization. We optimize the hyperparameter precision $\lambda$, using gradients while drawing samples from the posterior. We experimented with classification in MNIST with a 2-layer neural network where every layer has 50 neurons. Let us redefine the model with a hyperparameter $\lambda$ in the prior, which prior we assume that is gaussian. Then the prior is $p(\theta|\lambda)$ and the model is $p(x|\theta)p(\theta|\lambda)$. The log-likelihood in this case is $\log p(x|\theta) + \log p(\theta|\lambda)$. The goal is to maximize the log-likelihood $\max_{\theta,\lambda} \log p(x|\theta) + \log p(\theta|\lambda)$, where $\log p(\theta|\lambda) = \frac{\lambda}{2} |\theta|^2 -D \log (\lambda) + \text{const}$. The problem in this goal is that we have a singular solution $\lambda = \infty$ and $\theta=0$. We can achieve that in two steps. In step 1: run MCMC for convergence, given fixed $\lambda$ to generate $\theta_i$s. In step 2 we want to find the maximum $\lambda$ such as $\max_\lambda \sum_i \log p(\theta_i | \lambda)$.

The result of this experiment is presented in \cref{fig:hyperameter_optimization} where it shows the validation accuracy by maximizing the joint loss over the model parameters for various values of $\lambda$. With asterisks we can see the values of $\lambda$ that were automatically chosen by pSGLD and S$_d$-pSGLD.

\begin{figure}
	\centering 
	\includegraphics[width=1\linewidth]{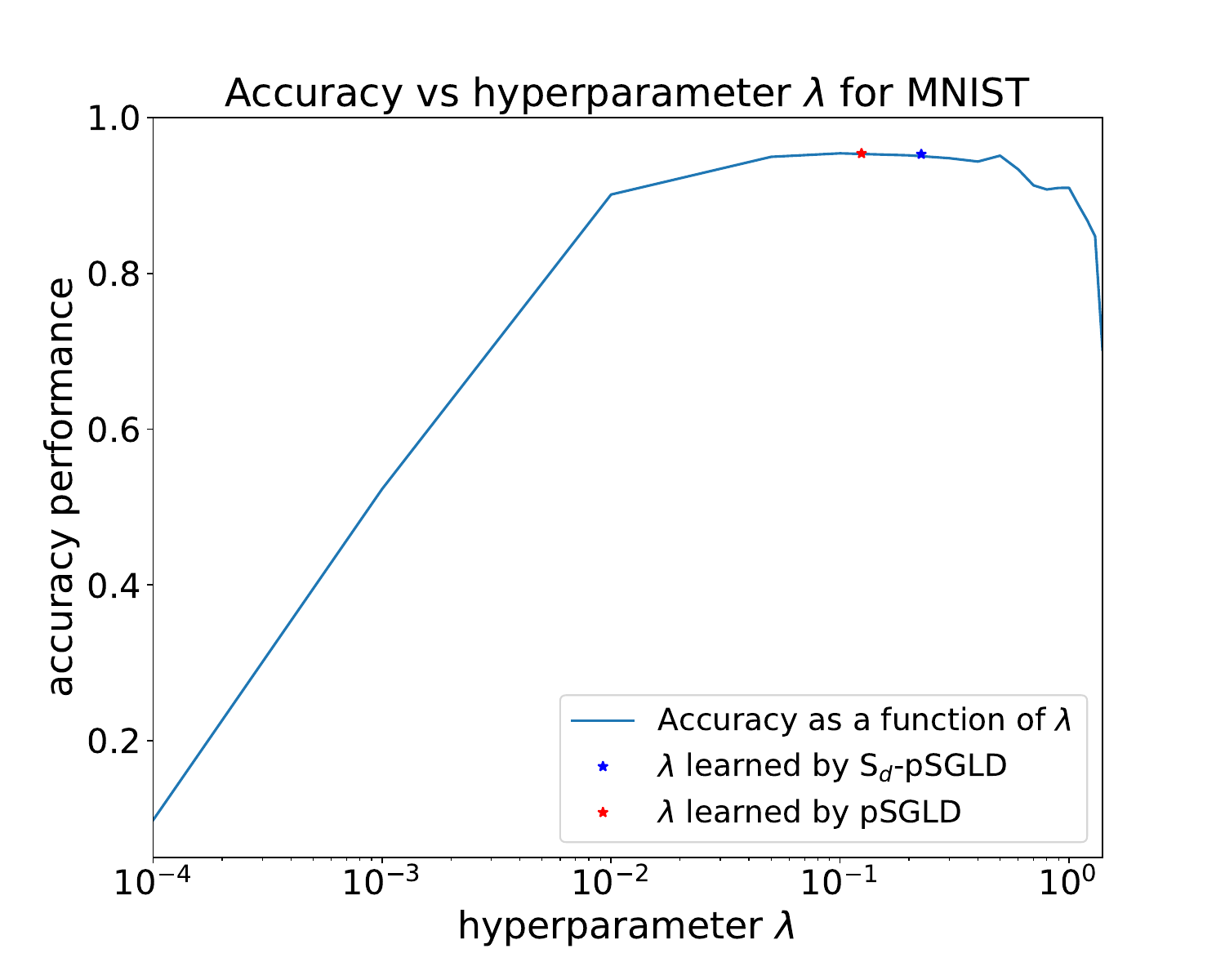}
    \vspace{-1em}
	\caption{Accuracy as a function of hyperparamter $\lambda$ on MNIST with a Fully Connected Nerwork. Stars show the values of $\lambda$ that were automatically chosen by pSGLD and S$_d$-pSGLD.}
	\label{fig:hyperameter_optimization}
\end{figure}

\section{Details on Experiments} \label{sec:exp_details}

\subsection{Qualitative Regression Experiments}

First, we aim to showcase qualitative differences in the empirical posterior distributions generated by a baseline SGMCMC algorithm and our proposed variants. To do so, we consider a regression task where 100 randomly sampled three-dimensional covariates $\{\vec{x}_i=[x_{i,1},x_{i,2},x_{i,3}]^T\}_{i=1,\dots,100}$ are used to sample response values $y_i\sim\mathcal{N}(\vec{w}^T\vec{x}_i + b, \sigma^2)$ where $\vec{w}=[w_1,w_2,w_3]^T=[1.5, -0.8, 1.3]^T$, $b=0.5$, and $\sigma^2=1$. More details on the generation process for $\vec{x}$ can be found in the Supplement.

We choose to fit a linear regression model of the same form as the generation process. $\sigma^2$ is assumed to be known. Thus, $\theta=[w_1,w_2,w_3,b]$. A standard normal distribution is used as the prior for each parameter. Due to conjugacy, the posterior distribution can be calculated analytically. As such, the MAP is roughly $\hat{\theta}_\text{MAP}\approx[0.52, 0.31, 0.47, 0.84]$.

The approximated posterior distributions for $\theta$ are found using SGLD, S-SGLD, and S$_d$-SGLD. For the latter two sampling schemes, two parameter partitions are tested: (i) two groups of parameters where $\theta_1=[w_1,w_2]$ and $\theta_2=[w_3,b]$ and (ii) four groups of parameters where $\theta_1=w_1, \theta_2=w_2, \theta_3=w_3,$ and $\theta_4=b$. For S$_d$-SGLD, $\rho=0.5$ and $K=4$ was used. 

The resulting posterior distributions for $(w_1,w_2)$ and $(w_1,w_3)$ from all five scenarios, with SGLD in the leftmost column as our baseline, can be seen in \cref{fig:lin_reg}. We observe that, as expected, correlations between $(w_1,w_2)$ still exist when they are allocated to the same parameter group and become apparently independent when assigned to different groups. We also note that the variance of the distributions shrink as the parameter space is partitioned into smaller groups. The underestimation of posterior variance is a commonly reported finding for VI techniques and is interesting to note that our non-parametric methods appear to exhibit this behavior as well. Finally, it appears that the S$_d$-SGLD adequately approximates S-SGLD with just slightly higher variances and very minor correlations between parameter groups being exhibited. 


\subsection{Real-World Data Experiments}

\paragraph{Framework Details.} In this subsection, we provide more detailed results for our experiments and a grid search for FMNIST, CIFAR10, and SVHN. We note that all the code apart from the metrics was written in PyTorch \citep{NEURIPS2019_9015}. Regarding the metrics, ESS was adopted from the TensorFlow probability library \citep{dillon2017tensorflow,abadi2016tensorflow} and IAC was calculated in python. For all the experiments, we used a seed of 2. Moreover, we note that we grouped the parameters in an ordered way for S$_d$-pSGLD and S-pSGLD. We denoted previously that $K\rho$ is the number of groups. So every parameter will go to the $i \mod K\rho$ group where $i$ is the parameter index. If, for instance, $K\rho$ is 8 then parameter 1 will go to group 1, parameter 2 will go to group 2, parameter 9 will go to group 1, etc. If $K\rho$ is the same as the number of parameters, every parameter will go into its own group. For the VI methods we used the official repository provided by \citet{ritter2021sparse}.\footnote{https://github.com/microsoft/bayesianize}

\paragraph{Metrics.} IAC is computed as $\tau_{f}=\sum_{\tau=-\infty}^{\infty} \rho_{f}(\tau)$ where $\rho_{f}$ is the normalized autocorrelation function of the stochastic process that generated the chain for $f$. $\rho_{f}$ is estimated via $\hat{\rho}_{f}(\tau)=\hat{c}_{f}(\tau) / \hat{c}_{f}(0)$; where $\hat{c}_{f}(\tau)=\frac{1}{N-\tau} \sum_{n=1}^{N-\tau}\left(f_{n}-\mu_{f}\right)\left(f_{n+\tau}-\mu_{f}\right)$ and $\mu_{f}=\frac{1}{N} \sum_{n=1}^{N} f_{n}$.\footnote{In practice, $\hat{c}_{f}(\tau)$ was calculated using a fast Fourier transform as it is more computationally efficient than directly summing.} ESS measures how many independent samples would be equivalent to a chain of correlated samples and is calculated as $n_{\mathrm{eff}}=\frac{n}{1+(n-1)p}$, where $n$ is the number of samples and $p$ is the autocorrelation.

\paragraph{MNIST.} Regarding MNIST, we ran all the experiments for 500 epochs with a batch size of 500 and a learning rate of 1e-2. For S$_d$-pSGLD, the $K$ is set to 300, which is the forward passes that the model does within 1 epoch. For the grouping of the parameters, for S$_d$-pSGLD we used group sizes of 2, 4, 8, 32, 128, 512, 2048, 4096, 8192, 16384, 32768, and 42200; and for S-pSGLD we used groups sizes of 2, 8, 32, 128, 512, 2048, 4096, and 8192.

\paragraph{FashionMNIST.} We ran all experiments for 500 epochs with a batch size of 500. For S$_d$-SGHMC, the $K$ is set to 2, which is the forward passes that the model does within 1 epoch. We observed with experimenting with $K$ that we do not need to set $K$ very high, and even a small number like 16 that we used here is enough to produce the same results as with an $K$ of 200 or 300. Regarding the parameter partitioning, for S$_d$-SGMCMC, we put every parameter in a different group, and for S-SGMCMC we used groups of 2, 4, 8, and 16. For S$_d$-pSGLD, pSGLD, S$_d$-SGHMC and SGHMC we tested their performances with learning rates of 1e-2, 1e-3, 1e-4, and 1e-5. We conducted a grid search for learning rate, dropout rate, and optimizers to find the best performing models and test them for their accuracy. The learning curves of the best models that we found are depicted in \cref{cifar10_svhn_acc}. In \cref{cifar10_svhn_acc}, S$_d$-pSGLD has $\rho=0.5$ and learning rate equal to 1e-3, pSGLD has learning rate equal to 1e-4, S$_d$-SGHMC has $\rho=0.5$ and learning rate equal to 1e-2 and SGHMC has learning rate equal to 1e-2.

\begin{table*}[!htbp]
\caption{Evaluation Metrics on FashionMNIST with pSGLD, S$_d$-pSGLD and S-pSGLD}
 \begin{tabular}{c c c c c c}
 \hline
  Method & dropout & LR & IAC & ESS & Accuracy\\ 
 \hline\hline
 S$_d$-pSGLD$_{|\theta|}$ & 0.1 & 1e-05 & 1018 & 5.63 & 0.918\\
 \hline
 S$_d$-pSGLD$_{|\theta|}$ & 0.1 & 1e-04 & 808 & 7 & 0.925\\
 \hline
  S$_d$-pSGLD$_{|\theta|}$ & 0.1 & 1e-03 & 754 & 7.48 & 0.924\\ 
 \hline
 S$_d$-pSGLD$_{|\theta|}$ & 0.1 & 1e-02 & 723 & 8.05 & 0.911\\  
 \hline
 S$_d$-pSGLD$_{|\theta|}$ & 0.5 & 1e-05 & 778 & 7.08 & 0.923\\   
 \hline
 S$_d$-pSGLD$_{|\theta|}$ & 0.5 & 1e-04 & 777 & 7.15 & 0.923\\ 
 \hline
 S$_d$-pSGLD$_{|\theta|}$ & 0.5 & 1e-03 & 737 & 7.57 & 0.925\\ 
 \hline
 S$_d$-pSGLD$_{|\theta|}$ & 0.5 & 1e-02 & 677 & 8.24 & 0.91\\ 
 \hline
 pSGLD & - & 1e-5 & 779 & 7.09 & 0.924\\ 
 \hline
 pSGLD & - & 1e-4 & 774 & 7.16 & 0.911\\ 
 \hline
 pSGLD & - & 1e-3 & 770 & 7.26 & 0.809\\ 
 \hline
 pSGLD & - & 1e-2 & 745 & 7.48 & 0.724\\ 
 \hline
 S-pSGLD$_2$ & - & 1e-3 & 740 & 7.55 & 0.918\\  
 \hline
 S-pSGLD$_4$ & - & 1e-3 & 751 & 7.45 & 0.919\\ 
 \hline
 S-pSGLD$_8$ & - & 1e-3 & 776 & 7.24 & 0.919\\ 
 \hline
 S-pSGLD$_{16}$ & - & 1e-3 & 855 & 6.64 & 0.916\\ 
 \hline
\end{tabular}
\centering
\label{fmnist_psgld}
\end{table*}

\begin{table*}[!htbp]
\caption{Evaluation Metrics on FashionMNIST with SGHMC, S$_d$-SGHMC and S-SGHMC}
 \begin{tabular}{c c c c c c}
 \hline
  Method & dropout & LR & IAC & ESS & Accuracy\\ 
 \hline\hline
 S$_d$-SGHMC$_{|\theta|}$ & 0.1 & 1e-05 & 782 & 7.08 & 0.412\\
 \hline
 S$_d$-SGHMC$_{|\theta|}$ & 0.1 & 1e-04 & 888 & 6.41 & 0.796\\
 \hline
  S$_d$-SGHMC$_{|\theta|}$ & 0.1 & 1e-03 & 793 & 6.98 &0.92\\ 
 \hline
 S$_d$-SGHMC$_{|\theta|}$ & 0.1 & 1e-02 & 1113 & 5.06 & 0.922\\  
 \hline
 S$_d$-SGHMC$_{|\theta|}$ & 0.5 & 1e-05 & 790 & 6.93 & 0.207\\   
 \hline
 S$_d$-SGHMC$_{|\theta|}$ & 0.5 & 1e-04 & 789 & 6.9 & 0.758\\ 
 \hline
 S$_d$-SGHMC$_{|\theta|}$ & 0.5 & 1e-03 & 796 & 6.81 & 0.92\\ 
 \hline
 S$_d$-SGHMC$_{|\theta|}$ & 0.5 & 1e-02 & 923 & 5.70 & 0.927\\ 
 \hline
 SGHMC & - & 1e-5 & 791 & 6.93 & 0.206\\ 
 \hline
 SGHMC & - & 1e-4 & 789 & 6.9 & 0.751\\ 
 \hline
 SGHMC & - & 1e-3 & 795 & 6.83 & 0.92\\ 
 \hline
 SGHMC & - & 1e-2 & 920 & 5.72 & 0.928\\ 
 \hline
 S-SGHMC$_2$ & - & 1e-2 & 928 & 5.67 & 0.928\\  
 \hline
 S-SGHMC$_4$ & - & 1e-2 & 915 & 5.77 & 0.927\\ 
 \hline
 S-SGHMC$_8$ & - & 1e-2 & 1142 & 4.87 & 0.919\\ 
 \hline
 S-SGHMC$_{16}$ & - & 1e-2 & 1121 & 4.92 & 0.906\\ 
 \hline
\end{tabular}
\centering
\label{fmnist_sghmc}
\end{table*}

\paragraph{CIFAR10.} The setup is similar to the one we used in FashionMNIST as we ran all experiments for 500 epochs with a batch size of 128. For S$_d$-SGHMC, the $K$ is set to 2, which $K$ is the forward passes that the model does within 1 epoch. Regarding the parameter partitioning, for S$_d$-SGMCMC, we put every parameter in a different group, and for S-SGMCMC we used groups of 2, 4, 8, and 16. For S$_d$-pSGLD, pSGLD, S$_d$-SGHMC and SGHMC we tested their performances with learning rates of 1e-2, 1e-3, 1e-4, and 1e-5. We conducted a grid search for learning rate, dropout rate, and optimizers to find the best performing models and test them for their accuracy. The learning curves of the best models that we found are depicted in \cref{cifar10_svhn_acc}. In \cref{cifar10_svhn_acc}, S$_d$-pSGLD has $\rho=0.5$ and learning rate equal to 1e-5, pSGLD has learning rate equal to 1e-5, S$_d$-SGHMC has $\rho=0.5$ and learning rate equal to 1e-3 and SGHMC has learning rate equal to 1e-3.

\begin{table*}[!htbp]
\caption{Evaluation Metrics on CIFAR10 with pSGLD, S$_d$-pSGLD and S-pSGLD}
 \begin{tabular}{c c c c c c c}
 \hline
  Method & dropout & LR & IAC & ESS & Accuracy\\ 
 \hline\hline
 S$_d$-pSGLD$_{|\theta|}$ & 0.1 & 1e-02 & 623 & 7.23 & 0.191\\
 \hline
 S$_d$-pSGLD$_{|\theta|}$ & 0.1 & 1e-03 & 572 & 7.6 & 0.896\\
 \hline
  S$_d$-pSGLD$_{|\theta|}$ & 0.1 & 1e-04 & 692 & 6.45 & 0.921\\ 
 \hline
 S$_d$-pSGLD$_{|\theta|}$ & 0.1 & 1e-05 & 922 & 4.88 & 0.922\\  
 \hline
 S$_d$-pSGLD$_{|\theta|}$ & 0.5 & 1e-02 & 546 & 8.01 & 0.768\\   
 \hline
 S$_d$-pSGLD$_{|\theta|}$ & 0.5 & 1e-03 & 582 & 7.88 & 0.918\\ 
 \hline
 S$_d$-pSGLD$_{|\theta|}$ & 0.5 & 1e-04 & 691 & 6.85 & 0.926\\ 
 \hline
 S$_d$-pSGLD$_{|\theta|}$ & 0.5 & 1e-05 & 620 & 7.22 & 0.925\\ 
 \hline
 pSGLD & - & 1e-2 & 716 & 8.01 & 0.666\\ 
 \hline
 pSGLD & - & 1e-3 & 740 & 7.87 & 0.866\\ 
 \hline
 pSGLD & - & 1e-4 & 780 & 7.41 & 0.914\\ 
 \hline
 pSGLD & - & 1e-5 & 831 & 6.89 & 0.924\\ 
 \hline
 S-pSGLD$_2$ & - & 1e-3 & 600 & 7.44 & 0.894\\  
 \hline
 S-pSGLD$_4$ & - & 1e-3 & 599 & 7.4 & 0.905\\ 
 \hline
 S-pSGLD$_8$ & - & 1e-3 & 709 & 6.41 & 0.881\\ 
 \hline
 S-pSGLD$_{16}$ & - & 1e-3 & 767 & 5.93 & 0.836\\ 
 \hline
\end{tabular}
\centering
\label{table:cifar_psgld}
\end{table*}

\begin{table*}[!htbp]
\caption{Evaluation Metrics on CIFAR10 with SGHMC, S$_d$-SGHMC and S-SGHMC}
 \begin{tabular}{c c c c c c c}
 \hline
  Method & dropout & LR & IAC & ESS & Accuracy\\ 
 \hline\hline
 S$_d$-SGHMC$_{|\theta|}$ & 0.1 & 1e-02 & 608 & 7.16 & 0.91\\
 \hline
 S$_d$-SGHMC$_{|\theta|}$ & 0.1 & 1e-03 & 975 & 4.6 & 0.922\\
 \hline
  S$_d$-SGHMC$_{|\theta|}$ & 0.1 & 1e-04 & 654 & 6.63 & 0.869\\ 
 \hline
 S$_d$-SGHMC$_{|\theta|}$ & 0.1 & 1e-05 & 652 & 6.65 & 0.724\\  
 \hline
 S$_d$-SGHMC$_{|\theta|}$ & 0.5 & 1e-02 & 584 & 7.7 & 0.918\\   
 \hline
 S$_d$-SGHMC$_{|\theta|}$ & 0.5 & 1e-03 & 751 & 6.26 & 0.925\\ 
 \hline
 S$_d$-SGHMC$_{|\theta|}$ & 0.5 & 1e-04 & 679 & 6.73 & 0.886\\ 
 \hline
 S$_d$-SGHMC$_{|\theta|}$ & 0.5 & 1e-05 & 772 & 6.01 & 0.778\\ 
 \hline
 SGHMC & - & 1e-2 & 727 & 7.94 & 0.86\\ 
 \hline
 SGHMC & - & 1e-3 & 832 & 6.84 & 0.924\\ 
 \hline
 SGHMC & - & 1e-4 & 862 & 6.57 & 0.885\\ 
 \hline
 SGHMC & - & 1e-5 & 858 & 6.6 & 0.746\\ 
 \hline
 S-SGHMC$_2$ & - & 1e-3 & 583 & 7.49 & 0.913\\  
 \hline
 S-SGHMC$_4$ & - & 1e-3 & 624 & 7.03 & 0.919\\ 
 \hline
 S-SGHMC$_8$ & - & 1e-3 & 904 & 4.97 & 0.908\\ 
 \hline
 S-SGHMC$_{16}$ & - & 1e-3 & 822 & 5.47 & 0.774\\ 
 \hline
\end{tabular}
\centering
\label{table:cifar_sghmc}
\end{table*}



\paragraph{SVHN.} We also ran all of the experiments for 500 epochs with a batch size of 128. Here for S$_d$-SGHMC, the $K$ is set to 2, which is the forward passes that the model does within 1 epoch. We note that $K$ here is less than on CIFAR10 and FashionMNIST, but as we mentioned before, this does not make a difference for our results, as we have tested. Regarding the parameter partitioning, for S$_d$-SGMCMC, we put every parameter in a different group, and for S-SGMCMC we used groups of 2, 4, 8, and 16. For S$_d$-pSGLD, pSGLD, S$_d$-SGHMC and SGHMC we tested their performances with learning rates of 1e-1, 1e-2, 1e-3, 1e-4, 1e-5, 1e-6. We conducted a grid search for learning rate, dropout rate, and optimizers to find the best performing models and test them for their accuracy. The learning curves of the best models that we found are depicted in \cref{cifar10_svhn_acc}. In \cref{cifar10_svhn_acc}, S$_d$-pSGLD has $\rho=0.5$ and learning rate of 1e-4, pSGLD has a learning rate of 1e-5, S$_d$-SGHMC has $\rho=0.5$ and learning rate equal to 1e-2 and SGHMC has learning rate equal to 1e-3.

\paragraph{VI methods.} For all VI methods, MFVI, MFVI-I and Ensemble-I we tested their performances with learning rates of 1e-2,1e-3 and 1e-4. We used a seed of 2 and 300 epochs for training. The rest of the settings are the same as the ones provided in the official repository. In table you can see the classification accuracy results for various learning rates.

\begin{table*}[!htb]
\centering
    \caption{Classification accuracy scores for VI methods on CIFAR10, SVHN and FMNIST.}
    \begin{subtable}{.3\linewidth}
      \centering
        \caption{CIFAR10}
        \begin{tabular}{lll}
         \hline
         \addlinespace[0.2cm]
          Method & LR & Accuracy\\ 
          \addlinespace[0.2cm]
         \hline
         \addlinespace[0.1cm]
         MFVI & 1e-2 & 80.06\\
         MFVI & 1e-3 & 91.52\\
         MFVI & 1e-4 & 87.5\\
         MFVI-I & 1e-2 & 91.73\\
         MFVI-I & 1e-3 & 91.91\\
         MFVI-I & 1e-4 & 87.6\\
         Ensemble-I & 1e-2 & 92.68\\
         Ensemble-I & 1e-3 & 91.96\\
         Ensemble-I & 1e-4 & 87.22\\
        \addlinespace[0.1cm]
        \hline
        \addlinespace[0.2cm]
        \end{tabular}
    \end{subtable}%
    \begin{subtable}{.3\linewidth}
      \centering
        \caption{SVHN}
        \begin{tabular}{lll}
         \hline
         \addlinespace[0.2cm]
          Method & LR & Accuracy\\ 
          \addlinespace[0.2cm]
         \hline
         \addlinespace[0.1cm]
         MFVI & 1e-2 & 94.79\\
         MFVI & 1e-3 & 94.36\\
         MFVI & 1e-4 & 94.19\\
         MFVI-I & 1e-2 & 95.91\\
         MFVI-I & 1e-3 & 95.47\\
         MFVI-I & 1e-4 & 94.78\\
         Ensemble-I & 1e-2 & 96.01\\
         Ensemble-I & 1e-3 & 95.89\\
         Ensemble-I & 1e-4 & 95.08\\
        \addlinespace[0.1cm]
        \hline
        \addlinespace[0.2cm]
        \end{tabular}
    \end{subtable} 
    \begin{subtable}{.3\linewidth}
      \centering
        \caption{FMNIST}
        \begin{tabular}{lll}
         \hline
         \addlinespace[0.2cm]
          Method & LR & Accuracy\\ 
          \addlinespace[0.2cm]
         \hline
         \addlinespace[0.1cm]
         MFVI & 1e-2 & 90.93\\
         MFVI & 1e-3 & 92.24\\
         MFVI & 1e-4 & 90.04\\
         MFVI-I & 1e-2 & 92.70\\
         MFVI-I & 1e-3 & 92.74\\
         MFVI-I & 1e-4 & 91.48\\
         Ensemble-I & 1e-2 & 93.45\\
         Ensemble-I & 1e-3 & 93.23\\
         Ensemble-I & 1e-4 & 91.54\\
        \addlinespace[0.1cm]
        \hline
        \addlinespace[0.2cm]
        \end{tabular}
    \end{subtable} 
\end{table*}

\end{document}